%% file: main.tex
\documentclass[11pt,letterpaper]{article}
\usepackage[margin=1in]{geometry}
\usepackage[ruled,linesnumbered]{algorithm2e}
\usepackage[utf8]{inputenc}
\usepackage{amsmath, nicefrac}
\usepackage{amsthm}
\usepackage{thmtools}
\usepackage{thm-restate}
\usepackage{hyperref}
\usepackage{natbib}

\usepackage[nottoc,numbib]{tocbibind}
\usepackage{amssymb,verbatim, indentfirst}
\usepackage{mathtools, bbm, xcolor, mathtools}

\usepackage{algorithmic}
\usepackage{svg}
\usepackage{float}
\newtheorem{theorem}{Theorem}[section]
\newtheorem{definition}[theorem]{Definition}

\newtheorem{lemma}[theorem]{Lemma}
\newtheorem{corollary}[theorem]{Corollary}

\newtheorem{remark}[theorem]{Remark}

\usepackage{graphicx}
\usepackage{booktabs}
\usepackage{url}
\usepackage[title]{appendix}
\usepackage[utf8]{inputenc} 
\usepackage{url}            
\usepackage{booktabs}       
\usepackage{amsfonts}       
\usepackage{nicefrac}       
\usepackage{microtype}      

\usepackage{amsmath}
\usepackage{bbm}
\usepackage{verbatim}
\usepackage[font=small]{caption}
\usepackage{amsthm}
\usepackage{subcaption}
\usepackage{xcolor}
\graphicspath{ {images/} }

\newcommand{\comb}{\text{comb}}
\usepackage{subfiles}

\DeclareMathOperator*{\E}{\mathbb{E}}

\let\Pr\relax\DeclareMathOperator*{\Pr}{\mathbf{Pr}}

\DeclareMathOperator*{\argmin}{arg\,min}

\DeclareMathOperator{\poly}{poly}

\hypersetup{
    colorlinks=true,
    linkcolor=blue,
    filecolor=magenta,      
    urlcolor=cyan,
}

\newcommand{\norm}[1]{\left\|#1\right\|}

\newcommand{\wh}{\widehat}
\newcommand{\R}{\mathbb{R}}
\newcommand{\Z}{\mathbb{Z}}
\newcommand{\D}{\mathcal{D}}

\newcommand{\eps}{\varepsilon}

\newcommand{\grad}{\nabla}

\newcommand{\wt}{\widetilde}

\usepackage{xifthen, cleveref}

\newcommand{\abs}[1]{\left|#1\right|}
\newcommand{\Prb}[2][]{ \ifthenelse{\isempty{#1}}
  {\Pr\left[#2\right]}
  {\Pr_{#1}\left[#2\right]} }
\newcommand{\Ex}[2][]{ \ifthenelse{\isempty{#1}}
  {\E\left[#2\right]}
  {\E_{#1}\left[#2\right]} }
\newcommand{\var}[2][]{ \ifthenelse{\isempty{#1}}
  {\mathbf{Var}\left[#2\right]}
  {\mathop{\mathbf{Var}}_{#1}\left[#2\right]} }

\renewcommand{\d}[1]{\ensuremath{\operatorname{d}\!{#1}}}

\newcommand{\TV}{\textup{\textsf{TV}}}
\newcommand{\NC}{\textup{\textsf{NC}}}

\newcommand{\zo}{\{0, 1\}}

\newcommand{\ReLU}{\text{ReLU}}

\newcommand{\switch}{\text{switch}}

\newcommand\blfootnote[1]{%
  \begingroup
  \renewcommand\thefootnote{}\footnote{#1}%
  \addtocounter{footnote}{-1}%
  \endgroup
}
\allowdisplaybreaks
\title{Diffusion Posterior Sampling is Computationally Intractable}
\author{Shivam Gupta\blfootnote{Authors listed in alphabetical order}\\UT Austin\\\texttt{shivamgupta@utexas.edu} \and Ajil Jalal\\ UC Berkeley\\\texttt{ajiljalal@berkeley.edu} \and Aditya Parulekar\\UT Austin\\\texttt{adityaup@cs.utexas.edu} \and Eric Price\\UT Austin\\\texttt{ecprice@cs.utexas.edu} \and Zhiyang Xun\\UT Austin\\\texttt{zxun@cs.utexas.edu}}

\begin{document}
\maketitle
\thispagestyle{empty}

\begin{abstract}


Diffusion models are a remarkably effective way of learning and sampling from a distribution $p(x)$.  In posterior sampling, one is also given a measurement model $p(y \mid x)$ and a measurement $y$, and would like to sample from $p(x \mid y)$.  Posterior sampling is useful for tasks such as inpainting, super-resolution, and MRI reconstruction, so a number of recent works have given algorithms to heuristically approximate it; but none are known to converge to the correct distribution in polynomial time.

In this paper we show that posterior sampling is \emph{computationally intractable}: under the most basic assumption in cryptography---that one-way functions exist---there are instances for which \emph{every} algorithm takes superpolynomial time, even though \emph{unconditional} sampling is provably fast. We also show that the exponential-time rejection sampling algorithm is essentially optimal under the stronger plausible assumption that there are one-way functions that take exponential time to invert.

\end{abstract}

\newpage

\thispagestyle{empty}

\tableofcontents

\newpage

\setcounter{page}{1}

\input{introduction}
\input{related_work}
\input{proof_overview_lower_bound}

\input{proof_overview_upper_bound}
\input{conclusion}

\section*{Acknowledgements}
We thank Xinyu Mao for discussions about the cryptographic assumptions.
SG, AP, EP and ZX are supported by NSF award CCF-1751040 (CAREER) and the NSF AI Institute for Foundations of Machine Learning (IFML). AJ is supported by ARO 051242-002.

\bibliographystyle{alpha}
\bibliography{references}

\appendix
\input{lower_bound_instance}
\input{lower_bound}
\input{lower_bound_relu_approx}
\input{lower_bound_putting_it_together}

\input{upper_bound}

\input{unconditional_sampler}
\input{crypto_hardness}
\input{utility}

\end{document}

%% file: introduction.tex
\section{Introduction}

Over the past few years, diffusion models have emerged as a powerful
way for representing distributions of images.  Such models, such as
Dall-E~\cite{ramesh2022hierarchical} and Stable Diffusion~\cite{rombach2022highresolution}, are very effective at learning
and sampling from distributions.  These models can then be used as
priors for a wide variety of downstream tasks, including inpainting,
superresolution, and MRI reconstruction.

Diffusion models are based on representing the \emph{smoothed scores}
of the desired distribution.  For a distribution $p(x)$, we define the
smoothed distribution $p_{\sigma}(x)$ to be $p$ convolved with
$\mathcal N(0, \sigma^2 I)$.  These have corresponding smoothed scores
$s_\sigma(x) := \grad \log p_{\sigma}(x)$.  Given the smoothed scores,
the distribution $p$ can be sampled using an SDE~\cite{ho2020denoising} or an
ODE~\cite{song2021denoising}.  Moreover, the smoothed score is the minimizer of
what is known as the \emph{score-matching} objective, which can be estimated
from samples.

Sampling via diffusion models is fairly well understood from a
theoretical perspective.  The sampling SDE and ODE are both fast
(polynomial time) and robust (tolerating $L_2$ error in the estimation
of the smoothed score).  Moreover, with polynomial training samples of
the distribution, the empirical risk minimizer (ERM) of the score
matching objective will have bounded $L_2$ error, leading to accurate
samples~\cite{Block2020GenerativeMW, gupta2023sampleefficient}.  So diffusion models give fast and robust unconditional
samples.

But sampling from the original distribution is not the main utility of
diffusion models: that comes from using the models to solve
downstream tasks.  A natural goal is to sample from the
\emph{posterior}: the distribution gives a prior $p(x)$ over images,
so given a noisy measurement $y$ of $x$ with known measurement model
$p(y \mid x)$, we can in principle use Bayes' rule to compute and
sample from $p(x \mid y)$. Often (such as for inpainting,
superresolution, MRI reconstruction) the measurement process is the
noisy linear measurement model, with measurement $y = Ax + \eta$ for a known measurement
matrix $A \in \R^{m \times d}$ with $m < d$, and Gaussian noise $\eta = \beta \mathcal N(0, I_m)$; we will focus on such
linear measurements in this paper.

Posterior sampling has many appealing properties for image
reconstruction tasks.  For example, if you want to identify $x$
precisely, posterior sampling is within a factor 2 of the minimum
error possible for \emph{every} measurement model and \emph{every}
error metric~\cite{mri_jalal}.  When ambiguities do arise, posterior
sampling has appealing fairness guarantees with respect to sensitive
attributes~\cite{Jalal2021FairnessFI}.

Given the appeal of posterior sampling, the natural question is: is
efficient posterior sampling possible given approximate smoothed
scores?  A large number of recent
papers~\cite{mri_jalal,chung2023diffusion,kawar2021snips,
  trippe2023practical, song2023pseudoinverseguided,
  kawar2022denoising, particle2024diffusion} have studied algorithms
for posterior sampling, with promising empirical results.  But all
these fail on some inputs; can we find a better posterior sampling
algorithm that is fast and robust in all cases?

There are several reasons for optimism.  First, there's the fact that
\emph{unconditional} sampling is possible from approximate smoothed
scores; why not posterior sampling?  Second, we know that
\emph{information-theoretically, it is possible}: rejection sampling of
the unconditional samples (as produced with high fidelity by the diffusion process) is
very accurate with fairly minimal assumptions.  The only problem is that rejection sampling is slow: you
need to sample until you get lucky enough to match on every
measurement, which takes time exponential in $m$.

And third, we know that the \emph{unsmoothed} score of the posterior
$p(x \mid y)$ is computable efficiently from the unsmoothed score of
$p(x)$ and the measurement model: $\nabla_x \log p(x\mid y) = \nabla\log p(x) + \nabla \log p(y\mid x)$.  This is sufficient to run Langevin
dynamics to sample from $p(x \mid y)$.  Of course, this has the same
issues that Langevin dynamics has for unconditional sampling: it can
take exponential time to mix, and is not robust to errors in the
score.  Diffusion models fix this by using the \emph{smoothed} score
to get robust and fast (unconditional) sampling.  It seems quite
plausible that a sufficiently clever algorithm could also get robust
and fast posterior sampling.

Despite these reasons for optimism, in this paper we show that \textbf{no fast posterior
  sampling algorithm exists}, even given good approximations to the
smoothed scores, under the most basic cryptographic assumption that one-way functions exist.  In fact, under the further assumption that some one-way function is exponentially hard to
invert, there exists a distribution---one for which the smoothed
scores are well approximated by a neural network so
that \emph{unconditional} sampling is fast---that takes exponential in $m$
time for posterior sampling.  Rejection sampling takes time 
exponential in $m$, and so, one can no longer hope for much general
improvement over rejection sampling.

\paragraph{Precise statements.}  To more formally state our results,
we make a few definitions.  We say a distribution is ``well-modeled''
if its smoothed scores can be represented by a polynomial size neural
network to polynomial precision:

\begin{definition}[$C$-Well-Modeled Distribution]
\label{def:well_modeled_distribution}
    For any constant $C > 0$, we say a distribution $p$ over $\R^d$ with covariance $\Sigma$ is ``$C$-well-modeled'' by score networks if $\norm{\Sigma} \lesssim 1$ and there are approximate scores $\wh{s}_\sigma$ that satisfy
    \[
    \E_{x\sim p_\sigma}[ \norm{\wh{s}_\sigma(x) - s_\sigma(x)}^2] < \frac{1}{d^C\sigma^2}
    \]
    and can be computed by a $\poly(d)$-parameter neural network with $\poly(d)$-bounded weights for every $\frac{1}{d^C} \le \sigma \le d^C$.

\end{definition}

Throughout our paper we will be comparing similar distributions.  We say distributions are $(\tau, \delta)$ close if they are close up to some shift $\tau$ and failure probability
$\delta$:

\begin{definition}[$(\tau, \delta)$-Close Distribution]
We say the distribution of $x$ and $\wh{x}$ are $(\tau, \delta)$ close if they can be coupled such that
\[
\Pr[\norm{x - \wh{x}} > \tau] < \delta.
\]
\end{definition}

An \emph{unconditional} sampler is one that is $(\tau, \delta)$ close to the true
distribution.

\begin{definition}[$(\tau, \delta)$-Unconditional Sampler]
    \label{def:unconditional_sampler}
    A $(\tau, \delta)$ unconditional sampler of a distribution $\D$ is one where its samples $\wh{x}$ are $(\tau, \delta)$ close to the true $x \sim \D$.
\end{definition}

The theory of diffusion models~\cite{chen2023sampling} says that the diffusion process gives
an unconditional sampler for well-modeled distributions that takes
polynomial time (with the precise polynomial improved by subsequent
work~\cite{benton2024nearly}).

\begin{restatable}[Unconditional Sampling for Well-Modeled Distributions]{theorem}{unconditionalsampler}
    \label{thm:unconditional_sampler}
    For an $O(C)$-well-modeled distribution $p$, the discretized reverse diffusion process with approximate scores gives a $\left(\frac{1}{d^{C}}, \frac{1}{d^{C}} \right)$-unconditional sampler (as defined in Definition~\ref{def:unconditional_sampler}) for any constant $C > 0$ in $\poly(d)$ time.
\end{restatable}


But what about \emph{posterior} samplers?  We want that, for most measurements $y$, the conditional distribution is $(\tau, \delta)$ close to the truth:


    

\begin{definition}[$(\tau, \delta)$-Posterior Sampler]
    \label{def:conditional_sampler}
Let $\mathcal{D}$ be a distribution over $X \times Y$ with density $p(x, y)$.    
    Let $\mathsf C$ be an algorithm that takes in $y \in Y$ and outputs samples from some distribution $\wh p_{\mid y}$ over $X$.  We say $\mathsf{C}$ is a $(\tau, \delta)$-Posterior Sampler for $\mathcal{D}$ if, with $1-\delta$ probability over $y \sim \mathcal{D}_Y$,  $\wh p_{\mid y}$ and $p(x \mid y)$ are $(\tau, \delta)$ close.
\end{definition}

  
  As described above, we consider the linear measurement model:

\begin{definition}[Linear Measurement Model]
    In the linear measurement model with $m$ measurements and noise parameter $\beta$, we have for $x \in \R^d$, the measurement $y = Ax + \eta$ for $A \in \R^{m \times d}$ normalized such that $\norm{A} \leq 1$, and $\eta = \beta \mathcal N(0, I_m)$.
\end{definition}

One way to implement posterior sampling is by rejection sampling.
As long as the measurement noise $\beta$ is much bigger than the error
$\tau= \frac{1}{\poly(d)}$ from the diffusion process, this is
accurate.  However, the running time is exponentially large in $m$:
\begin{restatable}[Upper Bound]{theorem}{upperbound}
\label{thm:upperbound}
    Let $C > 1$ be a constant. Consider an $O(C)$-well-modeled distribution and a linear measurement model with $\beta > \frac{1}{d^C}$. When $\delta > \frac{1}{d^C}$, rejection sampling of the diffusion process gives a $(\frac{1}{d^C}, \delta)$-posterior sampler that takes
    $\poly(d)(\frac{O(1)}{\beta\sqrt{\delta}})^m$ time.
\end{restatable}

Our main result is that
this is nearly tight: 
\begin{restatable}[Lower Bound]{theorem}{lowerbound}
\label{thm:lower_bound}
   Suppose that one-way functions exist.  Then for any $m > d^{0.01}$, there exists a $10$-well-modeled distribution over $\R^d$, and linear measurement model with $m$ measurements and noise parameter $\beta = \Theta(\frac{1}{\log^2d})$, such that   $(\frac{1}{10}, \frac{1}{10})$-posterior sampling requires superpolynomial time in $d$.
\end{restatable}

To be a one-way function, inversion must take superpolynomial time on average. It is widely believed, including for problems based on lattices~\cite{Lattice} and elliptic curves~\cite{zhandry2019magic}, that many one-way function candidates need exponential time to invert.
Under the stronger assumption that there exist some one-way functions that require exponential time to invert with non-negligible probability, we can show that posterior sampling takes $2^{\Omega(m)}$ time:



\begin{restatable}[Lower Bound: Exponential Hardness]{theorem}{lowerboundfinegrained}
\label{thm:lower_bound_fine_grained}
    Suppose that there exist one-way functions $f: \{\pm 1\}^m \to \{\pm 1\}^{m}$ that require $2^{\Omega(m)}$ time to invert.  Then for any $m \leq O(d)$ and $C > 1$, there exists a $C$-well-modeled distribution over $\R^d$ and linear measurement model with $m$ measurements and noise level $\beta = \frac{1}{C^2\log^2d}$, such that $(\frac{1}{10}, \frac{1}{10})$-posterior sampling takes at least $2^{\Omega(m)}$ time.
\end{restatable}


Assuming such strong one-way functions exist, then for the lower bound instance, $2^{\Omega(m)}$ time is
necessary and rejection sampling takes $2^{O(m \log \log d)} \poly(d)$
time.  Up to the $\log \log d$ factor, this shows that rejection
sampling is the best one can hope for in general.


\begin{remark}
   The lower bound produces a ``well-modeled" distribution, meaning that the scores are representable by a polynomial-size neural network, but there is no requirement that the network be shallow.  One could instead consider only shallow networks; the same theorem holds, except that $f$ must also be computable by a shallow depth network.  Many candidate one-way functions can be computed in $\NC^0$ (i.e., by a constant-depth circuit)~\cite{NC0}, so the cryptographic assumption is still mild. 
\end{remark}

%% file: related_work.tex
\section{Related Work}
Diffusion models~\cite{pmlr-v37-sohl-dickstein15, Dhariwal2021DiffusionMB, Song2019GenerativeMB} have emerged as the most popular approach to deep generative modeling of images, serving as the backbone for the recent impressive results in text-to-image generation~\cite{ramesh2022hierarchical,rombach2022highresolution}, along with state-of-the-art results in video~\cite{blattmann2023stable, ho2022video} and audio~\cite{kong2021diffwave,chen2021wavegrad} generation.

 Noisy linear inverse problems capture a broad class of applications such as image inpainting, super-resolution, MRI reconstruction, deblurring, and denoising. The empirical success of diffusion models has motivated their use as a data prior for linear inverse problems, \emph{without} task-specific training. There have been several recent theoretical and empirical works~\cite{mri_jalal,chung2023diffusion,kawar2021snips, trippe2023practical, song2023pseudoinverseguided, kawar2022denoising, particle2024diffusion} proposing algorithms to sample from the posterior of a noisy linear measurement. We highlight some of these approaches below.

\paragraph{Posterior Score Approximation.} One class of approaches~\cite{chung2023diffusion, kawar2021snips,  song2023pseudoinverseguided} \emph{approximates} the intractable posterior score $\grad \log p_t(x_t | y) = \grad \log p_t(x_t) + \grad \log p_t(y | x_t)$ at time $t$ of the reverse diffusion process, and uses this approximation to sample. Here, $y = Ax_0 + \eta$ is the noisy measurement of $x_0 \sim p_0$, where $p_t$ is the density at time $t$.  For instance, \citet{chung2023diffusion} proposes the approximation $\grad \log p_t(y | x_t) \approx \grad \log p(x | \E\left[x_0 | x_t \right])$, thereby incurring error quantified by the so-called \emph{Jensen gap}. \cite{song2023pseudoinverseguided} proposes an approximation based on the pseudoinverse of $A$, while \cite{kawar2021snips} proposes to use the score of the posterior wrt measurement $y_t$ of $x_t$.

\paragraph{Replacement Method.} Another approach, first introduced in the context of inpainting~\cite{Lugmayr2022RePaintIU}, replaces the observed coordinates of the sample with a noisy version of the observation during the reverse diffusion process. An extension was proposed for general noisy linear measurements~\cite{kawar2022denoising}. This approach essentially also attempts to sample from an approximation to the posterior.

\paragraph{Particle Filtering.} A recent set of works~\cite{trippe2023practical, trippe2023diffusion, particle2024diffusion} makes use of Sequential Monte Carlo (SMC) methods to sample from the posterior. These methods are guaranteed to sample from the correct distribution as the number of particles goes to $\infty$. Our paper implies a lower bound on the number of particles necessary for good convergence.  Assuming one-way functions exist, polynomially many particles are insufficient in general, so that these algorithms takes superpolynomial time; assuming some one-way function requires exponential time to invert, particle filtering requires exponentially many particles for convergence.

To summarize, our lower bound implies that these approaches are either approximations that fail to sample from the posterior, and/or suffer from prohibitively large runtimes in general.

%% file: proof_overview_lower_bound.tex
\section{Proof Overview -- Lower Bound}
In this section, we give an overview of the proof of our main Theorem~\ref{thm:lower_bound}, which states that there is some well-modeled distribution for which posterior sampling is hard. The full proof can be found in the Appendix. 

The core idea of our proof is that any general posterior sampler would imply an algorithm that can invert a one-way function. A one-way function is formally defined as follows:
\begin{definition}
\label{def:one_way_functions}
    A polynomial-time computable function $f : \{-1, 1\}^* \to \{-1, 1\}^{*}$ is one-way if, for any polynomial-time randomized algorithm $\mathcal{A}$, any constant $c > 0$, and all sufficiently large $n$, \[
        \Pr_{x \sim \mathcal{U}_n}\left[f(\mathcal{A}(f(x))) = f(x)\right] < n^{-c}
    \]
    where $\mathcal{U}_n$ is the uniform distribution over $\{-1, 1\}^n$.
\end{definition}

The function $f$ is defined on inputs of arbitrary length; for inputs of length $n$ it can be assumed to have some fixed polynomial output length $m(n)$.  

\paragraph{An initial attempt.}
Suppose we have a one-way
function $f: \{-1, 1\}^d \to \{-1, 1\}^d$, and consider the distribution
that is uniform over $(s, f(s)) \in \{-1, 1\}^{2d}$ for all
$s \in \{-1, 1\}^d$.  This distribution is easy to sample from
unconditionally: sample $s$ uniformly, then compute $f(s)$.  At the
same time, posterior sampling is hard: if you observe the last $d$
bits, i.e. $f(s)$, a posterior sample should be from $f^{-1}(f(s))$;
and if $f$ is a one-way function, finding any point in this support is
computationally intractable on average.

However, it is not at all clear that this distribution is well-modeled
as per Definition~\ref{def:well_modeled_distribution}; we would need to be able to
accurately represent the smoothed scores by a polynomial size neural
network.  The problem is that for smoothing levels
$1 \ll \sigma \ll \sqrt{d}$, the smoothed score can have nontrivial
contribution from many different $(s, f(s))$; so it's not clear one can
compute the smoothed scores efficiently.  Thus, while posterior
sampling is intractable in this instance, it's possible the hardness
lies in representing and computing the smoothed scores using a diffusion model, rather than in \emph{using}
the smoothed scores for posterior sampling.

However, for smoothing levels $\sigma \ll \frac{1}{\sqrt{\log d}}$, the smoothed scores \emph{are}
efficiently computable with high accuracy.  The smoothed distribution is a mixture of Gaussians with very little overlap, so rounding to a nearby Gaussian and taking its score gives very high accuracy.

To design a better lower bound, we modify the distribution to encode $f(s)$ differently: into the phase of the discretization of a Gaussian.  At large smoothing levels, a discretized Gaussian looks essentially like an undiscretized Gaussian, and the phase information disappears.  Thus at large smoothing levels, the distribution is essentially like a product distribution, for which the scores are easy to compute.  At the same time, conditioning on the observations still implies inverting $f$, so this is still hard to conditionally sample; and it's still the case that small smoothing levels are efficiently computable.

Based on the above, we define our lower bound instance formally in Section~\ref{sec:lower_bound_instance}. Then, in Section~\ref{sec:conditional_sampling_implies_inversion} we sketch a proof of Lemma~\ref{lem:lower_bound_plugging_params_generic}, which shows that it is impossible to perform accurate posterior sampling for our instance, under standard cryptographic assumptions. Section~\ref{sec:relu_approx} shows that our lower bound distribution is well-modeled by a small ReLU network, which means that the hardness is not coming merely from inability to represent the scores, and that \emph{unconditional} sampling is provably efficient. Finally, we put these observations together to show the theorem.

\input{proof_overview_lower_bound_instance}

\input{proof_overview_relu_approx}

%% file: proof_overview_lower_bound_instance.tex

\subsection{Lower Bound Instance}
\label{sec:lower_bound_instance}
We define our lower bound instance here formally.
    Let $w_\sigma(x)$ denote the density of a Gaussian with mean zero and standard deviation $\sigma$, and let $\comb_\eps$ denote the Dirac Comb distribution with period $\eps$, given by
    \begin{align*}
        \comb_\eps(x) = \sum_{k=-\infty}^\infty \delta(x - k \eps)
    \end{align*}
For any $b \in \{-1, 1\}$, let $\psi_b$ be the density of a standard Gaussian discretized to multiples of $\eps$, with phase either $0$ or $\frac{\eps}{2}$ depending on $b$:
\[
\psi_b(x) \propto w_1(x) \cdot \comb_{\eps}\left(x - \eps/2 \cdot \frac{1-b}{2}\right).
\]

\begin{restatable}[Unscaled Lower Bound Distribution]{definition}{unscaledlowerbounddist}
\label{def:unconditional_distribution}
    Let $f : \{\pm 1\}^d \to \{\pm 1\}^{d'}$ be a given function. For $R > 0$ and for any $s \in \{\pm 1\}^d$, define the product distribution $g_s$ over $x \in \R^{d+d'}$ such that
    \begin{align*}
        x_i &\sim w_1(x_i - R \cdot s_i) &\text{for } i \leq d\\
        x_i &\sim \psi_{f(s)_{i-d}} &\text{for } i > d.
    \end{align*}
    The unconditional distribution $g$ we consider is the uniform mixture of $g_s$ over $s \in \{\pm 1\}^d$.
    
\end{restatable}

We will have $d' = O(d)$ throughout. Figure \ref{fig:distribution} gives a visualization of $g_s$; the final distribution is the mixture of $g_s$ over uniformly random $s$.

For ease of exposition, we will also define a scaled version of our distribution $g$ such that its covariance $\Sigma$ has $\|\Sigma\| \lesssim 1$.
\begin{restatable}[Scaled Lower Bound Distribution]{definition}{scaledlowerbounddist}
    \label{def:unconditional_distribution_scaled}
    Let $\wt g(x) = R^{d+d'} g(R \cdot x)$ be the scaled version of the distribution with density $g$ defined in Definition~\ref{def:unconditional_distribution}. Similarly, let $\wt g_s = R^{d+d'} g_s(R \cdot x)$. 
\end{restatable}
The measurement process then takes sample $x \sim \wt g$ and computes $Ax + \eta$, where $\eta = \mathcal{N}(0, \beta^2 I_{d'} )$ and $A = \begin{pmatrix}
    0^{d'\times d} & I_{d'}
\end{pmatrix}.$
That is, we observe the last $d'$ bits of $x$, with variance $\beta^2$ Gaussian noise added to each coordinate. 

\begin{figure*}
    \centering
    \includegraphics[width=\textwidth]{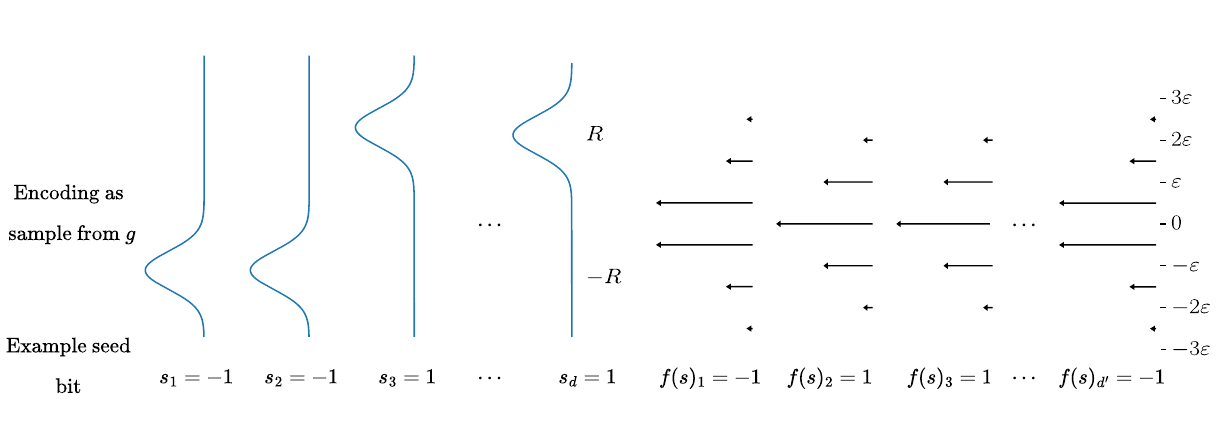}
    \caption{The distribution of each coordinate in $g_s$, has independent coordinates. For any seed $s \in \{\pm 1\}^d$, the first $d$ bits are normal distributions whose \textit{mean} is specified by $s_i$, and the last $d'$ bits are a discretized standard normal where the \textit{discretization} is specified by $f(s)_j$. The full distribution $g$ is a mixture over all seeds $s$ of $g_s$. }
    \label{fig:distribution}
\end{figure*}

\subsection{Posterior Sampling Implies Inversion} 
\label{sec:conditional_sampling_implies_inversion}
Below, we state the main result of this section, and give a sketch of the proof. We show that given any function $f : \R^d \to \R^m$, if we can conditionally sample the above measurement process, then we can invert $f$. For the sake of exposition, we assume here that $f$ has unique inverses; a similar argument applies in general. The full proof of this Lemma is given in the Appendix.

\begin{restatable}{lemma}{conditionalsampleimpliesinversion}
    For any function $f$, suppose $\mathsf C$ is an $(1/10, 1/10)$-posterior sampler in the linear measurement model with noise parameter $\beta$ for distribution with density $\wt g$ as defined in Definition~\ref{def:unconditional_distribution_scaled}, with $\eps \ge \beta \sqrt{32\log d}$ and $R \ge 32 \sqrt{\log d}$. If $\mathsf C$ takes time $T$ to run, then there exists an algorithm $\mathsf{A}$ that runs in time $T + O(d)$ such that 
    \[\Pr_{s, \mathsf{A}}[f(\mathsf{A}(f(s))) \neq f(s)] \le \frac{3}{4} \]
\end{restatable}



Take some $z \in \{\pm 1\}^{d'}$. Our goal is to compute $f^{-1}(z)$, using the posterior sampler for $\wt g$. To do this, we take a sample $\overline{z}_i \sim \psi_{z_{i}} * \mathcal{N}(0, \beta^2)$ for $i \in \{1, \dots, d'\}$, and feed in $\overline z$ into our posterior sampler, to output $\wh x$. We then take the first $d$ bits of $\wh x$, round each entry to the nearest $\pm 1$, and output the result. 



To see why this works, let's analyze what the resulting conditional distribution looks like. First, note that any sample $x \sim \wt g$ encodes some $(s, f(s))$ coordinate-wise so that the encoding of $f(s)$ is one of two discretizations of a normal distribution, with width $\eps$, offset by $\eps/2$ from each other (see Figure \ref{fig:distribution}). Furthermore, since $\beta \ll \eps$, these two encodings are distinguishable with high probability even after adding noise with variance $\beta^2$. Therefore, with high probability, our sample $\overline z$, which is a noised and discretized encoding of the input $z$ we want to invert, will be such that each coordinate is within $\eps/4$ of the correct discretization. Consequently, a posterior sample with this observation will correspond to an encoding of $(s, f(s))$ where $s = f^{-1}(z)$, with high probability. The first $d$ bits of this encoding are just the bits of $f^{-1}(z)$ smoothed by a gaussian with variance $1/R^2$, and since $R \gg 1$, rounding these coordinates to the nearest $\pm 1$ returns $f^{-1}(z)$, with high probability. 

So, we showed how to invert an arbitrary $f$ using a posterior sampler. The runtime of this procedure was just the runtime of the posterior sampler, along with some small overhead. In particular, if $f$ were a one-way function that takes superpolynomial time to invert, posterior sampling \emph{must} take superpolynomial time. Formally, we show the following:

\begin{restatable}{lemma}{lowerboundpluggingparamsgeneric}
\label{lem:lower_bound_plugging_params_generic}
    Suppose $m \ge d^{0.01}$ and one-way functions exist. Then, for $\wt g$ as defined in Definition~\ref{def:unconditional_distribution_scaled} with $\eps = \frac{1}{C \sqrt{\log d}}$ and $R = C \log d$, and linear measurement model with noise parameter $\beta = \frac{1}{C^2 \log^2 d}$ and measurement matrix $A \in \R^{m \times d}$, $(\frac{1}{10}, \frac{1}{10})$-posterior sampling takes superpolynomial time.
\end{restatable}

One minor detail is that a one-way function is defined to map $\{0, 1\}^n \to \{0, 1\}^{n'}$ for an unconstrained $n'$, while we want one that maps $\{0, 1\}^{d-m} \to \{0, 1\}^m$. Standard arguments imply that we can get such a function from the assumption; see Section~\ref{sec:crypto} for details. 








%% file: proof_overview_relu_approx.tex
\subsection{ReLU Approximation of Lower Bound Score}
\begin{figure*}[h]
\begin{subfigure}{.5\linewidth}
    \centering
    \includegraphics[width=\textwidth]{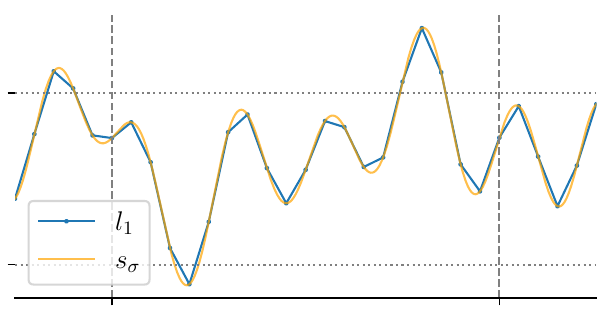}
    \caption{$l_1$ is unbounded and has an unbounded number of pieces}
\end{subfigure}%
\hfill
\begin{subfigure}{.5\linewidth}
    \centering
    \includegraphics[width=\textwidth]{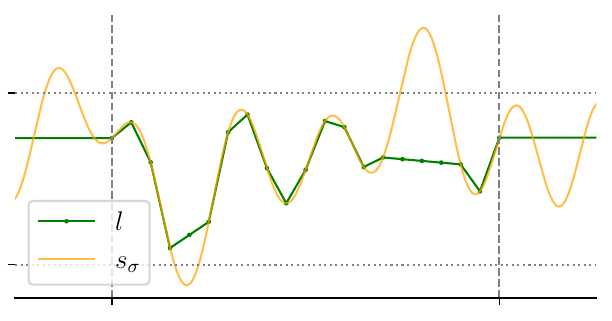}
    \caption{$l$ is bounded, has a small number of pieces.}
\end{subfigure}
\caption{Piecewise-Linear Approximations of Score $s_\sigma$}
\label{fig:piecewise_linear_approximations}
\end{figure*}

\label{sec:relu_approx}

We have shown that our (scaled) lower bound distribution $\wt g$ (as defined in Definition~\ref{def:unconditional_distribution_scaled}) is computationally intractable to sample from.
Now, we sketch our proof showing that $\wt g$ is well-modeled: the $\sigma$-smoothed scores are well approximated by a polynomially bounded ReLU network. The main result of this section is the following.
\begin{restatable}[Lower Bound Distribution is Well-Modeled]{corollary}{lowerboundwellmodeled}
\label{cor:g_is_well_modeled}
    Let $C$ be a sufficiently large constant. Given a ReLU network $f : \{\pm 1\}^d\to \{\pm 1\}^{d'}$ with $\poly(d)$ parameters bounded by $\poly(d)$ in absolute value, the distribution $\wt g$ defined in Definition~\ref{def:unconditional_distribution_scaled} for $R = C \log d$ and $ \frac{1}{\poly(d)} < \eps < \frac{1}{C \sqrt{\log d}}$, is $O(C)$-well-modeled.
\end{restatable}
To show this, we will first show that the unscaled distribution $g$ has a score approximation representable by a and polynomially bounded ReLU net. Rescaling by a factor of $R = C \log d$ then shows the above.

%

\paragraph{Notation.} We will let $h$ be the $\sigma$-smoothed version of $g$, and $h_r$ be the $\sigma$-smoothed version of $g_r$.
\paragraph{Strategy.} We will first show how to approximate the score of any $\sigma$-smoothed \emph{product} distribution using a polynomial-size ReLU network with polynomially bounded weights in our dimension $d$, $\frac{1}{\sigma}$ and $\frac{1}{\gamma}$ for $L^2$ error $\gamma^2$. 

Then, we will observe that when $\sigma$ is \emph{large}, so that $\poly(d) \ge \sigma \gg \eps \sqrt{\log d}$, $h$ becomes very close to a mixture of $(1 + \sigma^2) I_{d+d'}$-covariance Gaussians placed at the vertices of a scaled hypercube (in the first $d$ coordinates). Since this is a product distribution, we can represent its score using our ReLU construction.

On the other hand, when $\sigma$ is \emph{small}, for $R \gg \log d$ and $\frac{1}{\poly(d)} \le \sigma \ll \frac{R}{\sqrt{\log d}}$, the score of $h$ at any point $x$ is well approximated by the distribution $h_r$, where $r \in \{\pm 1\}^d$ represents the orthant containing the first $d$ coordinates of $x$. Since $h_r$ is a product distribution, our ReLU construction applies. 

Finally, we set $R \gg \log d$ so that for any $\frac{1}{\poly(d)} \le \sigma \le \poly(d)$, there is a polynomially bounded ReLU net that approximates the score of $h$. We now describe each of these steps in more detail.
\subsubsection{ReLU Approximation for Score of Product Distribution}
We will show first how to construct a ReLU network approximating the score of a \emph{one-dimensional} distribution -- the construction generalizes to product distributions in a straightforward way.

Consider any one-dimensional distribution $p$ with $\sigma$-smoothed version $p_\sigma$, and corresponding score $s_\sigma$. Suppose $p_\sigma$ has standard deviation $m_2$. We will first construct a \emph{piecewise-linear} function $l$ that approximates $s_\sigma$ in $L^2$.  

Since $s_\sigma$ is $\sigma$-smoothed, its value does not change much in \emph{most} $\sigma$-sized regions. More precisely, Lemma~\ref{lem:score_derivative_second_moment} shows that 
\begin{align*}
    \E_{x \sim p_\sigma}\left[\sup_{|c| \le \sigma} s_\sigma'(x+c)^2 \right] \lesssim \frac{1}{\sigma^4}
\end{align*}
This immediately gives a piecewise linear-approximation $l_1$ with $O(\gamma \sigma^2)$-width pieces: By Taylor expansion, we can write any $s_\sigma(x) = s_\sigma(\alpha_x) + (x - \alpha_x) s_\sigma'(\xi)$ for some $\xi$ between $\alpha_x$ and $x$. Then, if $\alpha_x$ is the largest discretization point smaller than $x$ (so that $|x - \alpha_x| \lesssim \gamma \sigma^2$), this gives that 
\begin{align*}
    \E\left[ (s_\sigma(x) - s_\sigma(\alpha_x))^2 \right] &\lesssim \gamma^2 \sigma^4 \E[\sup_{c}s_\sigma'(x+c)^2 ] \lesssim \gamma^2
\end{align*}
So, we can approximate every $s_\sigma(x)$ with $s_\sigma(\alpha_x)$, yielding a piecewise-\emph{constant} approximation. Then, we can similarly obtain another piecewise-constant approximation by replacing $s_\sigma(x)$ with $s_\sigma(\beta_x)$ for $\beta_x$ the smallest discretization point larger than $x$. By convexity, we can linearly interpolate between $s_\sigma(\alpha_x)$ and $s_\sigma(\beta_x)$ to obtain our piecewise-\emph{linear} approximation $l_1$ (see Fig.~\ref{fig:piecewise_linear_approximations}).

Unfortunately, $l_1$ suffers from two issues: 1) It is potentially unbounded, and 2) It has an unbounded number of pieces.

For 1), since $s_\sigma$ is $\sigma$-smoothed, it is bounded by with high probability, so that we can ensure that our approximation is also bounded without increasing its error much. For 2), since $p_\sigma$ has standard deviation $m_2$, Chebyshev's inequality gives that the total probability outside a radius $\frac{m_2}{\gamma \sigma^2}$ region is small, so that we can use a constant approximation outside this region. This allows us to bound the number of pieces by $\poly\left(\frac{m_2}{\gamma \sigma} \right)$, yielding our final approximation $l$.

As is well-known, such a piecewise linear function can be represented using a ReLU network with $\poly\left(\frac{m_2}{\gamma \sigma} \right)$ parameters, and each parameter bounded by $\poly\left(\frac{m_2}{\gamma \sigma} \right)$ in absolute value. For product distributions, we simply construct ReLU networks for each coordinate individually, and then append them, for bounds polynomial in $d$ and $1/\sigma$, $1/\gamma$ and $m_2$. In the remaining proof, whenever this construction is used, all these parameters are set to polynomial in $d$, for final bounds $\poly(d)$.
\subsubsection{ReLU Approximation for Large $\sigma$}

\begin{figure}[H]
\centering
\hspace{-0.2cm}
\begin{subfigure}{.5\linewidth}
    \centering
    \includegraphics[width=1.0\textwidth]{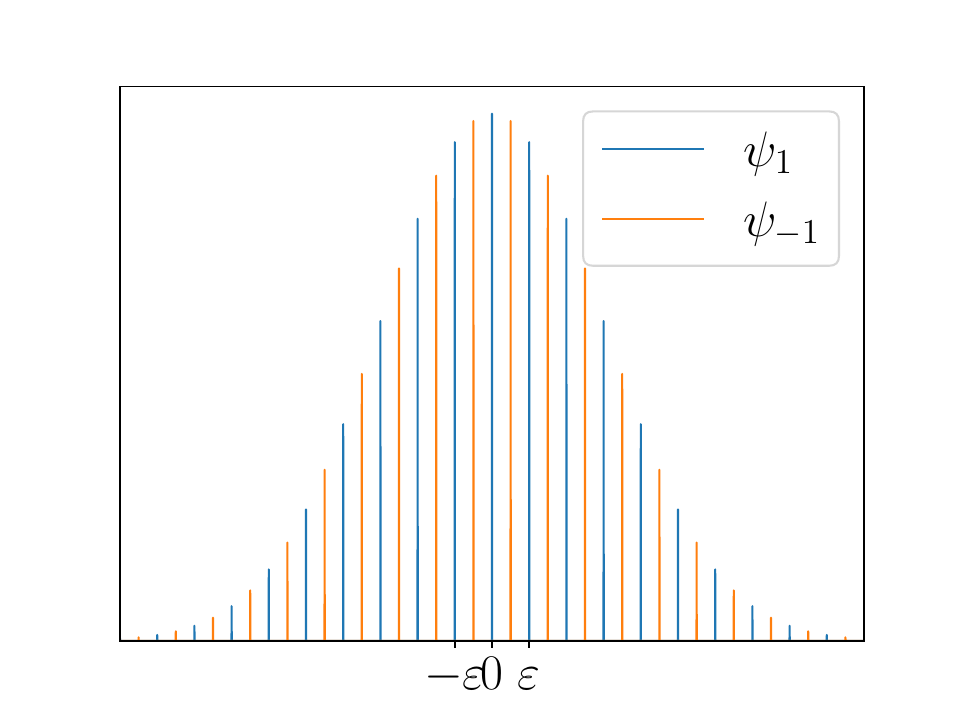}
\end{subfigure}%
\begin{subfigure}{.5\linewidth}
    \centering
    \includegraphics[width=1.0\textwidth]{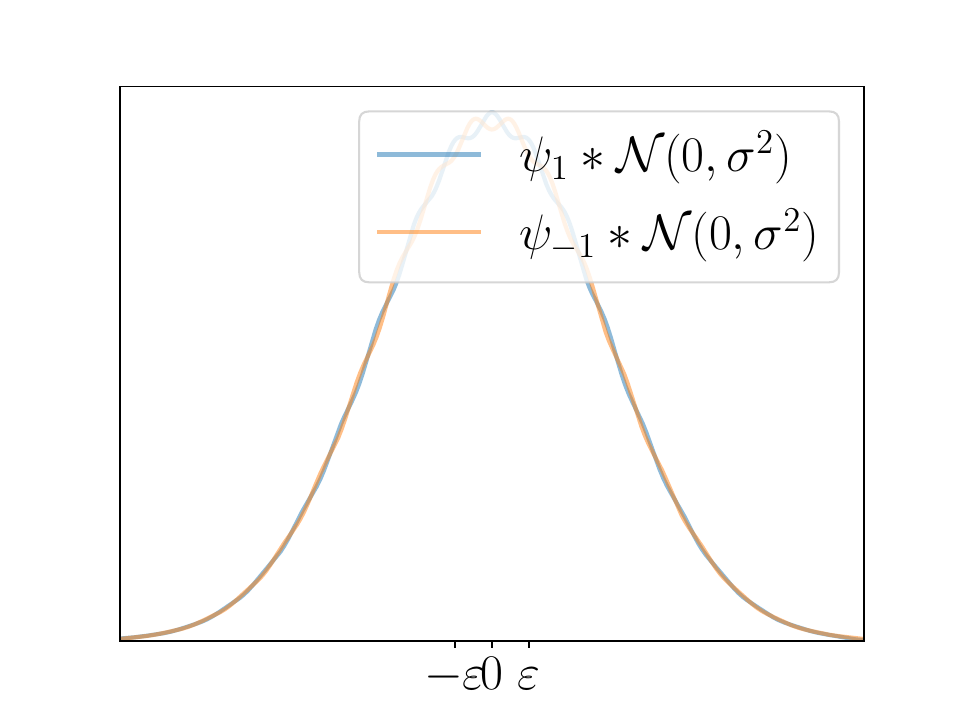}
\end{subfigure}
\caption{$\psi_1$ and $\psi_{-1}$ are discretized Gaussians with discretization width $\eps$ and phase $0$ and $\eps/2$ respectively. If we convolve with $\mathcal N(0, \sigma^2)$, we get a distribution close to Gaussian when $\sigma \ge \eps$ for each of $\psi_1, \psi_{-1}$.}
\label{fig:psismoothed}
\end{figure}
\vspace{-0.3cm}
Note that our lower bound distribution $g$ is such that the first $d$ coordinates are simply a mixture of Gaussians placed on the vertices of a (scaled) hypercube, while the last $d'$ coordinates are discretized Gaussians $\psi_1$ or $\psi_{-1}$, with the choice of discretization depending on the first $d$ coordinates. 

The only reason $g$ is not already a product distribution is that $\psi_1$ and $\psi_{-1}$ are different.  But for smoothing $\sigma \gg \eps\sqrt{\log d}$, a Fourier argument shows that the smoothed versions of $\psi_1$ and $\psi_{-1}$ are polynomially close to each other. See Figure~\ref{fig:psismoothed} for an illustration.



\subsubsection{ReLU Approximation for Small $\sigma$}
When $\sigma\ll \frac{R}{\sqrt{\log d}}$ and $R \gg \log d$, consider the density $h(x)$ for $x_{1, \dots,d}$ lying in the orthant identified by $r \in \{\pm 1\}^d$. Recall that
\begin{align*}
    h(x) = \frac{1}{2^d} \sum_{s \in \{\pm 1\}^d} h_s(x)
\end{align*}
where $h_s$ is the product distribution that is Gaussian with mean $R \cdot s_i$ in the first $d$ coordinates and is a smoothed discretized Gaussian with mean $0$ in the remaining $d'$ coordinates.

We first show that $h(x)$ is approximated by $\frac{h_r(x)}{2^d}$ up to small additive error. This is because every $h_s$ has radius at most $\sqrt{1 + \sigma^2} \lesssim \frac{R}{\sqrt{\log d}}$ and there are $\approx \binom{d}{k}$ points $s \neq r$ with the mean of $h_s$ at least $\Omega(\sqrt{k} R)$ away from $x$. So, the total contribution of all the terms involving $h_s(x)$ to $h(x)$ for $s \neq r$ is at most $\approx \frac{1}{2^d} \cdot \frac{1}{\poly(d)}$. We can show that $\grad h(x)$ is approximated by $\frac{\grad h_r(x)}{2^d}$ in $L^2$ up to similar additive error in an analogous way.

We then show that the score of $h_r$ serves as a good approximation to the score of $h$ for all such points $x$ such that $x_{1, \dots, d}$ lies in the orthant identified by $r$. For $x$ close to the mean of $h_r$ (to within $R/10$, say), the above gives that $h(x)$ is approximated up to \emph{multiplicative} error by $\frac{h_r(x)}{2^d}$, and $\grad h(x)$ is approximated up to multiplicative error by $\frac{\grad h_r(x)}{2^d}$. Together, this gives that the \emph{score} of $h$ at $x$, $\frac{\grad h(x)}{h(x)}$ is approximated by the score of $h_r$ at $x$ up to $\frac{1}{\poly(d)}$ error. On the other hand, for $x$ far from the mean of $h_r$, since the density itself is small, the total contribution of such points to the score error is negligible.

Since the score of $h$ is well-approximated by the score of $h_r$, and $h_r$ is a \emph{product} distribution, we can essentially use our ReLU construction for product distributions to represent its score, after using a small gadget to identify the orthant that $x_{1, \dots, d}$ lies in.

\subsection{Putting it all Together}
Lemma~\ref{lem:lower_bound_plugging_params_generic} shows that it is computationally hard to sample from $\wt g$ from the posterior of a noisy linear measurement when $f$ is a one-way funciton, while Corollary~\ref{cor:g_is_well_modeled} shows that $\wt g$ has score that is well-modeled by a ReLU network when $f$ can be represented by a polynomial-sized ReLU network. In Section~\ref{sec:crypto}, we show that any one-way function can be represented using a polynomial-sized ReLU network. Thus, together, these imply our lower bound, Theorem~\ref{thm:lower_bound}.

Essentially the same argument holds under the stronger guarantee that there exists a one-way function that takes exponential time to invert, for a lower bound \emph{exponential} in the number of measurements $m$.

%% file: proof_overview_upper_bound.tex
  \section{Proof Overview - Upper Bound}
  \begin{algorithm}
\caption{Rejection Sampling Algorithm }\label{alg:rej_sampling}
\begin{algorithmic}[1] 
\REQUIRE $y \in Y$
\WHILE{True}
  \STATE Sample $x \sim {\mathcal{D}}_x$
  \STATE Compute $q := e^\frac{-\|Ax - y\|^2}{2\beta^2}$  (proportional to $p(y \mid x)$)
  \STATE Generate a random number $r \sim U(0, 1)$
  \IF{$r < q$}
      \STATE \textbf{return} $x$
  \ENDIF
\ENDWHILE
\end{algorithmic}
\end{algorithm}
  
  In this section, we sketch the proof of Theorem~\ref{thm:upperbound} in Section~\ref{sec:upper_bound_proof}: the time complexity of posterior sampling by rejection sampling (Algorithm~\ref{alg:rej_sampling}). For ease of discussion, we only consider the case when $\delta = \Theta(1)$.
  The proof overview below will repeatedly refer to events as occurring with ``arbitrarily high probability"; this means the statement is true for every constant probability $p < 1$.  (Usually there will be a setting of constants in big-O notation nearby that depends on $p$.)
  
  

  
  \paragraph{Sampling Correctness With Ideal Sampler.} 
  To illustrate the idea of the proof, we first focus on the scenario where we can sample from the distribution of $x$ perfectly. We aim to show that rejection sampling perfectly samples $x \mid y$. 
   To prove the correctness of Algorithm 1, noting that each round is independent, it suffices to verify that each round outputs $x$ with probability density proportional to $p(x \mid y)$.  We have \[
      p(x \mid y) = \frac{p(y \mid x) p(x)}{p(y)} \propto p(y \mid Ax) p(x) 
      \propto e^{-\frac{\norm{Ax - y}}{2\beta^2}} p(x).
  \]
  Therefore, with a perfect unconditional sampler for ${\mathcal{D}_x}$ (sampling $x$ according to density $p(x)$), rejection sampling perfectly samples $x \mid y$.

  \paragraph{Running time.}
  Now we show that for linear measurements $y = Ax + \beta \mathcal{N}(0, I_d)$, with arbitrarily high probability over $x\sim \D$, the acceptance probability per round is at least $\Theta(\beta)^m$; this implies the algorithm terminates in $(O(1) / \beta)^m$ rounds with arbitrarily high probability.
  For a given $y$, the acceptance probability per round is equal to 
  \begin{align*}
      \Ex[x]{e^{-\frac{\norm{Ax - y}^2}{2 \beta^2}}} 
      \ge \Prb[x]{\norm{Ax - y} \le O(\beta\sqrt{m})} \cdot e^{-O(m)}.
  \end{align*}
  
  We first focus on the case when $m = 1$. We aim to show that with arbitrarily high probability over $y$, 
  \[
    \Prb[x]{\norm{Ax - y} \le O (\beta)} \ge \beta.
  \] 
  
  For well-modeled distributions, the covariance matrix of $x$ has constant singular values. Then with arbitrarily high probability, $x$ is $O(1)$ in each direction. Since every singular value of $A$ is at most 1, the projection $Ax$ onto $\R$ will lie in $[-C, +C]$ for some constant $C$ with arbitrarily high probability. 
  
  We divide $[-C, +C]$ into $N = \frac{2C}{\beta}$ segments of length $\beta$, forming set $S$. 
  Now we only need to prove that with arbitrarily high probability over $y$, there exists a segment $\theta \in S$ satisfying for all $x \in \theta$, $|x - y| \le O(\beta)$ , and $
    \Pr_{x \sim {\mathcal{D}_x}}[Ax \in \theta] \gtrsim \beta.
  $
  For any constant $c > 0$, define
  \[
        S' := \{\theta \in S \mid \Pr_{x \sim {\mathcal{D}_x}}[Ax \in \theta] > \frac{c}{N}\}.
    \]
    Each segment in $S'$ has probability at least $\Omega(1 / N) \gtrsim \beta$ to be hit. Therefore, we only need to prove that, with arbitrarily high probability, $y = Ax + \eta$ satisfies these two independent events simultaneously: (1) $Ax$ lands in some segment $\theta \in S'$; (2) $\eta \lesssim \beta$.

    By a union bound, the probability that $Ax$ lies in a segment in $S \setminus S'$ is at most $N \cdot \frac{c}{N} \le c$. 
    For sufficiently small $c$, combining with the fact that $Ax \in S$ with arbitrarily high probability, we have (1) with arbitrarily high probability.
    Since that $\eta \sim \mathcal{N}(0, \beta^2)$. By the concentration of Gaussian distribution, (2) is satisfied with arbitrarily high probability.   

  For the general case when $m > 1$, with arbitrarily high probability, $Ax$ will lie in $\{x \in \R^m \mid \norm{x} \le C\sqrt{m}\}$ for some $C > 0$. Instead of segments, we use $N = (\frac{O(1)}{\beta})^m$ balls with radius $\beta$ to cover $\{x \in \R^m \mid \norm{x} \le C\sqrt{m}\}$. 
  Following a similar argument, we can prove that with arbitrarily high probability over $y$, \[
    \Prb[x]{\norm{Ax - y} \le O(\beta\sqrt{m})} \ge \Theta(\beta)^m.
  \]


  \paragraph{Diffusion as unconditional sampler.}
  In practice, we do not have a perfect sampler for ${\mathcal{D}_x}$. Theorem~\ref{thm:unconditional_sampler} states that for $O(C)$-well-modeled distributions, diffusion model gives an unconditional sampler that samples from approximation distribution $\wh{{\mathcal{D}}}_x$ satisfying that there exists 
   a coupling between $x \sim {\mathcal{D}_x}$ and $\wh x \sim \wh{{\mathcal{D}}}_x$ such that with arbitrarily high probability, $\|x - \wh{x}\| \le 1 / d^{2C}$. 


 For $(x, \wh{x})$ drawn from this coupling, we know from our previous analysis that rejection sampling based on $x$ is correct.  But the algorithm only knows $\wh{x}$, which changes its behavior in two ways: (1) it chooses to accept based on $p(y \mid \wh{x})$ rather than $p(y \mid x)$, and (2) it returns $\wh{x}$ rather than $x$ on acceptance.  The perturbation from (2) is easily within our tolerance, since it is $\frac{1}{d^{2C}}$ close to $x$ with arbitrarily high probability.  

  For (1), we can show when $x$ and $\wh{x}$ are close, these two probabilities are nearly the same. 
  When $\norm{x - \wh x} \le \frac{1}{d^{2C}} \le o(\beta / \sqrt{m})$, we have\[
    \abs{\log \frac{p(y \mid \wh{x})}{p(y \mid x)}} = \abs{\frac{\norm{Ax - y}^2}{2\beta^2} -   \frac{\norm{A\wh{x} - y}^2}{2\beta^2}} \le o(1).
  \]
  This implies that $p(y \mid \wh{x}) = (1 \pm o(1)) p(y \mid x)$ and proves Theorem~\ref{thm:upperbound}.
  


%% file: conclusion.tex
\section{Conclusion and Future Work}  We have shown that one cannot hope
for a fast general algorithm for posterior sampling from diffusion
models, in the way that diffusion gives general guarantees for
unconditional sampling.  Rejection sampling, slow as it may be, is
about the fastest one can hope for on some distributions.  However,
people run algorithms that attempt to approximate the posterior
sampling every day; they might not be perfectly accurate, but they
seem to do a decent job.  What might explain this?

Given our lower bound, a positive theory for posterior sampling of
diffusion models must invoke distributional assumptions on the data.
Our lower bound distribution is derived from a one-way function, and
not very ``nice''.  It would be interesting to identify distributional
properties under which posterior sampling is possible, as well as
new algorithms that work under plausible assumptions.

%% file: lower_bound_instance.tex
\section{Lower Bound instance}
We first define our Lower Bound Distribution $g$ (up to scaling).
Let $w_\sigma(x)$ denote the density of a Gaussian with mean zero and standard deviation $\sigma$, and let $\comb_\eps$ denote the Dirac Comb distribution with period $\eps$, given by
    \begin{align*}
        \comb_\eps(x) = \sum_{k=-\infty}^\infty \delta(x - k \eps)
    \end{align*}
For any $b \in \{-1, 1\}$, let $\psi_b$ be the density of a standard Gaussian discretized to multiples of $\eps$, with phase either $0$ or $\frac{\eps}{2}$ depending on $b$:
\[
\psi_b(x) \propto w_1(x) \cdot \comb_{\eps}\left(x - \eps/2 \cdot \frac{1-b}{2}\right).
\]
\unscaledlowerbounddist*
We define our final Lower Bound distribution below, which is a scaled version of $g$.
\scaledlowerbounddist*

%% file: lower_bound.tex
\section{Lower Bound -- Posterior Sampling implies Inversion of One-Way Function}
\subsection{Notation}

Let $l \coloneqq [d] = \{1, 2, 3, \dots, d\}$, and let $r \coloneqq \{d+1, d+2, \dots, d+d'\}$, so that for any $x \in \R^{d+d'}$, $x_{[:d]}\in \R^{d}$ is a vector containing the first $d$ entries of $x$, and $x_{[-d':]}\in \R^{d'}$ is a vector containing the last $d'$ entries of $x$.

Let $\text{Round}_R : \R^k \to \{\pm R\}^k$ be such that for all $i \in [k]$,
    \[\text{Round}_R(x)_i = \argmin_{v \in \{\pm R\}} \abs{x_i - v}.\]
Let $\text{parity} : \Z \to \{-1, +1\}$ be such that $\text{parity}(2i) = -1, \text{parity}(2i+1) = 1$ for all $i \in Z$. Let $\text{Bits}_\eps : \R^{k} \to \{\pm 1\}^{k}$ be such that for all $i \in [k]$, 
\[(\text{Bits}_\eps(y))_i = \text{parity}\left(\argmin_{i \in \Z} \abs{i\cdot \frac{\eps}{2} - y_i}\right)\]
This function takes a value $y$ and returns a guess for whether $y$ comes from a smoothed distribution discretized to even multiples of $\eps/2$ or odd multiples of $\eps/2$, based on which is closer.

\begin{definition}[Conditional Distribution]
    Let $g$ be the distribution defined in \ref{def:unconditional_distribution}, parameterized by a function $f$, and real values $R, \eps > 0$. For some noise pdf $h$, we define $\mathcal{X}^h_{f, R, \eps}$ to be the distribution over $(x, y)$ where $x \sim g$ and $y \sim x_{[-d':]} + h$. 
\end{definition}

We also explicitly define the two noise models we will be using for the lower bound: we take \begin{equation}
\label{def:unconditional_distribution_x_notation}
    \mathcal{X}^\beta_{f, R, \eps} \coloneqq \mathcal{X}^{w_\beta}_{f, R, \eps}, \qquad w_\beta = N(0, \beta^2).
\end{equation} Let $(\mathcal{X}^\beta_{f, R, \eps})_y$ denote the marginal over $y$. Further, $\mathcal{X}^{\beta, \beta_\text{max}}_{f, R, \eps} \coloneqq \mathcal{X}^{b}_{f, R, \eps} $ where $b$ is a clipped normal distribution: $b \coloneqq \text{clip}(\beta_\text{max}, N(0, \beta^2))$. 

\subsection{Inverting $f$ via Posterior Sampling}
\begin{lemma}
\label{lem:conditional_distribution_mixture_of_bitstrings}
Let $\beta_\text{max} \le \eps/4$ and $\sqrt{32\log\frac{d}{\delta}} \le R$. Then, 
    \[\Pr_{x^b, y^b \sim \mathcal{X}^{\beta, \beta_\text{max}}_{f, R, \eps}}\left[f(\text{Round}_R(x^b_{[:d]})) = \text{Bits}_\eps(y^b)\right] \ge 1 - \delta\]
\end{lemma}
\begin{proof}
    Let $x^b, y^b \sim \mathcal{X}^{\beta, \beta_\text{max}}_{f, R, \eps}$. By definition, we know that $y_b \sim x^b_{[-d':]} + \text{clip}(\beta_\text{max}, N(0, \beta^2)$. Further, for all indices $i$, $(x^b_{[-d':]})_i = j\eps/2$ for some integer $j$. So, if $\beta_\text{max} \le \eps/4$, then \begin{equation}\label{eq:bits1}\text{Bits}_\eps(y^b) = \text{Bits}_\eps(x^b_{[-d':]}).\end{equation} We know that $x^b$ is drawn from a uniform mixture over $g_s(x)$, as defined in \ref{def:unconditional_distribution}. So, fixing an $s \in \{\pm 1\}^d$. We have that \begin{equation}\label{eq:bits2}\text{Bits}_\eps(x^b_{[-d':]}) = s.\end{equation} On the other hand, $x_{[:d]}$ is a product of gaussians centered at $Rs_i$ in the $i$th coordinate. Therefore, for all $i < d$,
    \[\Pr_{x^b}\left[\abs{(x^b_{[:d]})_i - Rs_i} \le \sqrt{2\log\frac{d}{\delta}}\right] \ge 1-\frac{\delta}{d}\]
    Since $\sqrt{2\log\frac{d}{\delta}} \le R/4$, we get that 
    \begin{equation}\label{eq:bits3}\Pr_{x^b} \left[\text{Round}_R(x^b_{[:d]}) =s\right] \ge 1-\delta.\end{equation}
    Putting together \cref{eq:bits1}, \cref{eq:bits2}, and \cref{eq:bits3} we get 
    \[\Pr_{x^b} \left[\text{Round}_R(x^b_{[:d]}) = \text{Bits}_\eps(y^b)\right] \ge 1-\delta\]
    
\end{proof}
 \begin{lemma}
 \label{lem:bits_close_to_round}
 Let $\mathsf{C}$ be a $(\tau, \delta)$-conditional sampling algorithm for $\mathcal{X}^\beta_{f, R, \eps}$. If $\eps \ge \beta \sqrt{32\log \frac{d}{\delta}}$, $\tau \le R/4$, and $32 \log \frac{d}{\delta} \le R^2$, then for $y \sim (\mathcal{X}^\beta_{f, R, \eps})_y$ and $\wh{x} \sim \mathsf{C}(y)$,
 \begin{align*}
    \Pr[f(\text{Round}_R(\wh{x}_{[:d]})) \neq \text{Bits}_\eps(y)] \le 5\delta.
 \end{align*}
 \end{lemma}
 \begin{proof}
     Let $\mathcal{X}^\beta_{f, R, \eps}$ have pdf $p^\beta$. Assume we have a $(\tau, \delta)$-posterior sampler over $\mathcal{X}^\beta_{f, R, \eps}$ that outputs sample from distribution $\wh{\mathcal{X}}$ with distribution $\wh p$. 
     This means that with probability $1-\delta$ over $y$, there exists a coupling $\mathcal{P}$ over $(x, \wh x)$ such that $(x, \wh x)$ are $(\tau, \delta)$-close. Therefore, there exists a distribution $\mathcal{P}$ over $(x, \wh x, y) \in \R^{d+d'} \times \R^{d+d'} \times \R^{d'}$ with density $p^\mathcal{P}$ such that $p^\mathcal{P}(x , y) = p^\beta(x , y)$, $p^\mathcal{P}(\wh x \mid y) = \wh p(\wh x\mid y)$, and 
    \[\Pr_{x, \wh x \sim \mathcal{P}} \left[\Vert x - \wh x\Vert_2 \le \tau\right] \ge 1-2\delta.\]
Now, let $\mathcal{X}^{\beta, \beta_\text{max}}_{f, R, \eps}$ have pdf $p^{\beta, \beta_\text{max}}$, with $\beta_\text{max} = \beta \sqrt{2 \log \frac{1}{\delta}}.$ We have \[TV(\mathcal{X}^\beta_{f, R, \eps}, \mathcal{X}^{\beta, \beta_\text{max}}_{f, R, \eps}) \lesssim e^{-\beta_\text{max}^2/2\beta^2} \le \delta\]
Therefore, building on $\mathcal{P}$, we can construct a new distribution $\mathcal{P}'$ over $(x, \wh x, x^b, y, y^b) \in \R^{d+d'} \times \R^{d+d'} \times \R^{d+d'} \times \R^{d'} \times \R^{d'}$ with density $p^\mathcal{P'}$ such that $p^\mathcal{P'}(x , y) = p^\beta(x , y)$, $p^\mathcal{P'}(\wh x \mid y) = \wh p(\wh x\mid y)$, $p^\mathcal{P'}(x^b , y^b) = p^{\beta, \beta_\text{max}}(x^b ,y^b)$, $(x, y) = (x^b, y^b)$ with probability $1-\delta$, and 
    \[\Pr_{x, \wh x \sim \mathcal{P'}} \left[\Vert x - \wh x\Vert_2 \le \tau\right] \ge 1-2\delta\]
Therefore, under this distribution, 
\[\Pr_{\wh x, x^b \sim \mathcal{P'}} \left[\Vert \wh x - x^b\Vert_2 \le \tau\right] \ge 1-3\delta\]
In particular, we apply the fact that $\Vert \wh x_{[:d]} - x^b_{[:d]} \Vert_\infty \le \Vert \wh x - x^b \Vert_\infty \le \Vert \wh x - x^b \Vert_2$ to get 
\begin{equation}\label{eq:first_d_infty}\Pr_{\wh x, x^b \sim \mathcal{P'}} \left[\Vert \wh x_{[:d]} - x^b_{[:d]} \Vert_\infty \le \tau\right] \ge 1-3\delta.\end{equation}

Now, by the definition of $\mathcal{X}^{\beta, \beta_\text{max}}_{f, R, \eps}$, for all $i<d$, $x^b_i$ is a mixture of variance 1 normal distributions centered at $\pm R$. So, for any $i < d$, 
\[\Pr_{x^b} \left[\abs {x^b_i - \text{Round}_R(x^b_i)} \ge \sqrt{2\log\frac{d}{\delta}}\right] \le \frac{\delta}{d}\]
Applying a union bound over $i \in [d]$ and putting this together with \cref{eq:first_d_infty},
\[\Pr_{\wh x, x^b \sim P'} \left[\Vert \wh x_{[:d]}- \text{Round}_R(x^b_{[:d]})\Vert_\infty \le \sqrt{2\log\frac{1}{\delta}} + \tau\right] \ge 1-4\delta\]
So, since $\sqrt{2\log \frac{d}{\delta}} + \tau \le \frac{R}{4} + \frac{R}{4} = \frac{R}{2}$, and $\text{Round}_R((x^b_{[:d]})_i) \in \pm R$, we have 
\begin{align*}\Pr_{\wh x, x^b \sim P'} \left[\Vert \text{Round}_R(\wh x_{[:d]}) - \text{Round}_R(x^b_{[:d]})\Vert_\infty \le \sqrt{2\log\frac{d}{\delta}} + \tau\right] \le 1-3\delta\end{align*}
Again, the output of $\text{Round}_R$ is always $\pm R$, so this means
\begin{align*}
 \Pr_{\wh x, x^b \sim P'} \left[\text{Round}_R(\wh x_{[:d]}) = \text{Round}_R(x^b_{[:d]})\right] \ge 1-3\delta\end{align*}
Now, by Lemma \ref{lem:conditional_distribution_mixture_of_bitstrings}, since $\beta_\text{max} < \eps/4$ and $R \ge \sqrt{32 \log \frac{d}{\delta}}$, we have
\[\Pr_{x^b, y^b \sim P'}\left[f(\text{Round}_R(x^b_{[:d]})) = \text{Bits}_\eps(y^b)\right] \ge 1 - \delta\]
Therefore, 
\[\Pr_{\wh x, y^b \sim P'}\left[f(\text{Round}_R(\wh x_{[:d]})) = \text{Bits}_\eps(y^b)\right] \ge 1 - 4\delta\]
Finally, we know that $y = y^b$ with probability $1-\delta$. Therefore, we get
\[\Pr_{\wh x, y \sim P'}\left[f(\text{Round}_R(\wh x_{[:d]})) = \text{Bits}_\eps(y)\right] \ge 1 - 5\delta\]
 \end{proof}
\begin{theorem}
\label{thm:inverting_with_conditional_sampling}
    For any function $f$, let $\mathsf{C}$ be a $\left(R/4, \delta\right)$-posterior sampler (\ref{def:conditional_sampler}) for $\mathcal{X}_{f, R, \eps}^{\beta}$, as defined in \eqref{def:unconditional_distribution_x_notation}, with $\eps \ge \beta \sqrt{32\log \frac{d}{\delta}}$, and $R \ge  \sqrt{32\log \frac{d}{\delta}}$, that takes time $T$ to run. Then, there exists an algorithm $\mathsf{A}$ that runs in time $T + O(d)$ such that 
    \[\Pr_{s, \mathsf{A}}[f(\mathsf{A}(f(s))) \neq f(s)] \le 6\delta \]
\end{theorem}
\begin{proof}
Sample $y \sim h_{f(r)}$, where 
\[h_{s}(y) = \begin{cases}
    \left(w_1(y) \cdot \comb_\eps(y))\right) * N(0, \beta^2), & \text{$s_{i} = 1$}\\
            \left(w_1(y) \cdot \comb_\eps\left(y - \frac{\eps}{2}\right)\right) * N(0, \beta^2) & \text{$s_{i} = -1$} 
\end{cases}\]
Now, since $\beta \le \frac{\eps}{\sqrt{32 \log \frac{d}{\delta}}}$, each coordinate of the noise, drawn from $N(0, \beta^2)$, is less than $\eps/4$ with probability $1-\delta/d$. Therefore, 
\[\Pr\left[\text{Bits}_\eps(y) = f(r)\right] \ge 1 - \delta\]
By definition, $h_{s}$ is the same as the density of $(\mathcal{X}^\beta_{f, R, \eps})_y$. So, by Lemma \ref{lem:bits_close_to_round}, since $R \ge \tau/4, R \ge\sqrt{32\log\frac{d}{\delta}}$, and we take $\hat{x} \sim \mathsf{C}(y)$, we have
\[\Pr_{\wh x, y}\left[f(\text{Round}_R(\wh x_{[:d]})) \neq \text{Bits}_\eps(y)\right] \le 5\delta\]
Therefore, 
\[\Pr_{\wh x, y}\left[f(\text{Round}_R(\wh x_{[:d]})) \neq f(r)\right] \le 6\delta\]
So, our algorithm $\mathsf{A}$ can output $\text{Round}_R(\wh x_l)$. All we had to do to run this algorithm was to sample $d$ normal random variables, and then run our posterior sampler. This takes $T + O(d)$ time. 
\end{proof}

\conditionalsampleimpliesinversion*
\begin{proof}
    This follows from Theorem \ref{thm:inverting_with_conditional_sampling}, using the fact that after rescaling down by $R$, $\mathcal{X}^\beta_{f, R, \eps}$ as a distribution over $(x, y)$ is the same distribution as $x \sim \wt g$, with $y = Ax + N(0, \beta^2)$.
\end{proof}


\subsection{Inverting a One-Way function via Posterior Sampling}

\lowerboundpluggingparamsgeneric*
\begin{proof}
    When $m > d/2$, we can add an arbitrary number of dummy observations which always observes 0.   Posterior sampling in this instance is identical to only observing the first $d/2$ coordinates. Therefore, we only need to consider the case when $m \le d / 2$.
    
    When $d^{0.01} < m < d/2$, $d$ and $m$ are only polynomially separated. So, by \ref{lem:one_way_function_stretching}, we can construct a one-way function $f : \{\pm 1\}^{d-m} \to\{\pm 1\}^m$. By definition, we can see that $\wt g$, with measurement noise $\beta$ is the same distribution as $\mathcal{X}^{\beta R}_{f, R, \eps}$, scaled down by $R$. Therefore, by Theorem \ref{thm:inverting_with_conditional_sampling}, since $R \ge 32 \sqrt{\log \frac{d}{\delta}}$, $\eps \ge \beta R \sqrt{\log \frac{d}{\delta}}$, if we can run a posterior sampler in time $T$, we can invert $f$ with probability $1-6\delta$ in time $T+O(m)$. So, if $f$ takes time superpolynomial in $m$ to invert, then $T + O(m)$ is superpolynomial. Since $m > d^{0.01}$, this means that $T$ itself is superpolynomial in $d$. 
\end{proof}
\begin{lemma}
\label{lem:lower_bound_plugging_params}
    Suppose that there exist one-way functions $f: \{\pm 1\}^m \to \{\pm 1\}^{m}$ that require $2^{\Omega(m)}$ time to invert. Then, for any $m = O(d)$, for $\wt g$ as defined in Definition~\ref{def:unconditional_distribution_scaled} with $\eps = \frac{1}{C \sqrt{\log d}}$ and $R = C \log d$, and linear measurement model with noise parameter $\beta = \frac{1}{C^2 \log^2 d}$ and measurement matrix $A \in \R^{m \times d}$, $(\frac{1}{10}, \frac{1}{10})$-conditional sampling takes at least $2^{\Omega(m)}$ time.
\end{lemma}
\begin{proof}
    Similar to the proof of \cref{lem:lower_bound_plugging_params_generic}, we only need to consider the case when $m \le d / 2$.
    By definition, we can see that $\wt g$, with measurement noise $\beta$ is the same distribution as $\mathcal{X}^{\beta R}_{f, R, \eps}$, scaled down by $R$. Therefore, by Theorem \ref{thm:inverting_with_conditional_sampling}, since $R \ge 32 \sqrt{\log d}$, $\eps \ge \beta R \sqrt{\log d}$, if we can run a posterior sampler in time $T$, we can invert $f$ with probability $0.4$ in time $T+O(m)$. So, if $f$ takes at least time $2^{\Omega(m)}$ to run, then we must have $T + O(m) \ge 2^{\Omega(m)}$, which means $T \ge 2^{\Omega(m)}$. 
\end{proof}

%% file: lower_bound_relu_approx.tex
\section{Lower Bound -- ReLU Approximation of Score}
\input{piecewise_linear_approx}
\input{small_noise_level_lemmas}
\input{relu_approx_of_product}

\input{relu_approx_small_smoothing}
\input{smoothing_discretized_gaussian}
\input{large_noise_level_lemmas}
\input{relu_approx_large_smoothing}
\input{relu_approx_unconditional}
\input{alpha_well_behaved}

%% file: piecewise_linear_approx.tex
\subsection{Piecewise Linear Approximation of $\sigma$-smoothed score in One Dimension}
In this section, we analyze the error of a piecewise linear approximation to a smoothed score. We first show that for one dimensional distributions, we can get good approximations, and later extend it to product distributions in higher dimensions.

First, we show that a piecewise linear approximation that discretizes the space into intervals of width $\gamma$ has low error.

\begin{lemma}
\label{lem:linear_approximation}
    Let $p$ be a distribution over $\R$, and let $p_\sigma = p * N(0, \sigma^2)$ have score $s_\sigma$. Let $\gamma \le \sigma$, and let $S_i = [i\gamma, (i+1)\gamma)$ for all $i \in \Z$. Define a piecewise linear function $f:\R\to\R$ so that: for all $x$, if $i$ is such that $S_i \ni x$, then 
    \[f(x) = \frac{((i+1)\gamma-x)\cdot s(i\gamma) + (x-i\gamma) \cdot s((i+1)\gamma)}{\gamma}.\] Then $f$ is continuous and satisfies  
    \begin{align*}
        \E\left[(s(x) - f(x))^2 \right] \lesssim \frac{\gamma^2}{\sigma^4}
    \end{align*}
\end{lemma}
\begin{proof}
Define the left and right piecewise constant approximations $l(x) = s(i\gamma), r(x) = s((i+1)\gamma)$ for all $x \in S_i$.

We know that for any $y \in S_i$, there is some $y' \in [i\gamma, y]$ such that $s(y) = s(i\gamma) + (y-i\gamma)s'(y')$. So, we get \[\forall y \in S_i, \,s(y)\le s(i\gamma) + \gamma \sup_{z \in S_i} s'(z) \le s(i\gamma) + \gamma \sup_{\abs{c} \le \gamma} s'(y + c).\] Therefore,
\begin{align*}
    \E_{x \sim p}\left[(s_\sigma(x) - l(x))^2\right] \le \gamma^2 \E_{x\sim p} [\sup_{|c| \le \gamma} s'(y+c)^2] \lesssim \frac{\gamma^2}{\sigma^4}
\end{align*}
By Lemma \ref{lem:score_derivative_second_moment}. The same holds for $r(x)$. Now, recall that $f$ satisfies 
\[\forall i\in \Z,\,\forall x \in S_i,\,  f(x) = \frac{(i+1)\gamma-x}{\gamma}\cdot s(i\gamma) + \frac{x-i\gamma}{\gamma} \cdot s((i+1)\gamma).\]

The coefficients $\frac{(i+1)\gamma-x}{\gamma}$ and $\frac{x-i\gamma}{\gamma}$ sum to $1$ and are within the interval $[0, 1]$. So, at each point, $f$ is just a convex combination of the two approximations $l$ and $r$. Therefore, by convexity, for any $S_i$, if $x \in S_i$,
\begin{align}\E_{x \in S_i}[(s_\sigma(x) - f(x))^2] \le \E_{x \in S_i}[(s_\sigma(x) - l(x))^2] + \E_{x \in S_i}[(s_\sigma(x) - r(x))^2]\end{align}
This immediately gives us that 
\begin{align*}\E[(s_\sigma(x) - f(x))^2] \le \E[(s_\sigma(x) - l(x))^2] + \E[(s_\sigma(x) - r(x))^2]\lesssim \frac{\varepsilon^2}{\sigma^4}\end{align*}
Within each interval, the function is linear and so it is continous. We just need to check continuity at the endpoints. However, we can see that for any $i\in \Z$, $\lim_{x \to i\gamma^-} = \lim_{x \to i\gamma^+} = s(i\gamma)$, and so we also have continuity. 
\end{proof}
Unfortunately, the above approximation has an infinite number of pieces. To handle this, we show that in regions far away from the mean, a zero-approximation is good enough, given that the distribution has bounded second moment $m_2$.
\begin{lemma}
\label{lem:large_x_score}
Let $p$ be some distribution over $\R$ with mean $\mu$, and let $p_\sigma = p*N(0, \sigma^2)$ have score $s_\sigma$. Let $m_2^2 \coloneqq \E_{x\sim p}\left[(x-\mu)^2\right]$ be the second moment of $p_\sigma$. Further, let $\abs{\varphi} \le \frac{1}{\sigma}\log\frac{1}{\delta}$ be some constant. Then,
    \begin{align*}
        \E\left[(s_\sigma(x)-\varphi)^2 \cdot \mathbbm{1}_{|x - \mu| > \frac{m_2}{\sqrt{\delta}}} \right] \lesssim \frac{\sqrt{\delta}}{\sigma^2}
    \end{align*}
\end{lemma}
\begin{proof}
    We have
    \begin{align*}
        \E\left[(s_\sigma(x)-\varphi)^2 \cdot \mathbbm{1}_{|x - \mu| > \frac{m_2}{\sqrt{\delta}}} \right] 
        &\lesssim \E\left[s_\sigma(x)^2\cdot \mathbbm{1}_{|x - \mu| > \frac{m_2}{\sqrt{\delta}}}\right] + \E\left[\varphi^2 \cdot \mathbbm{1}_{|x - \mu| > \frac{m_2}{\sqrt{\delta}}} \right] 
    \end{align*}
    First, by Chebyshev's inequality, we know that 
    \[\Pr\left[\abs{x-\mu} \ge \frac{m_2}{\delta}\right] \le \delta\]
    Now, we use Cauchy Schwarz to bound the first term:
    \begin{align*}
        \E\left[s_\sigma(x)^2 \cdot 1_{|x - \mu| > \frac{m_2}{\sqrt{\delta}}} \right] &\le \sqrt{\E\left[s_\sigma(x)^4\right]\E\left[1_{|x - \mu| > \frac{m_2}{\sqrt{\delta}}} \right] }
        \\&= \sqrt{\E\left[s_\sigma(x)^4\right]\Pr\left[\abs{x - \mu} \ge \frac{m_2}{\sqrt{\delta}}\right]}
        \\&= \sqrt{\E\left[s_\sigma(x)^4\right]\cdot\delta}\lesssim \sqrt{\delta/\sigma^4} = \sqrt{\delta}/\sigma^2
    \end{align*}
    where the last line is by Corollary \ref{cor:score_interval_moment_bound_simplified}.
    Finally, for the second term, we know that 
    \begin{align*}
        \E\left[\varphi^2 \cdot \mathbbm{1}_{|x - \mu| > \frac{m_2}{\sqrt{\delta}}} \right] 
        &\le \E\left[\frac{1}{\sigma^2}\log^2\frac{1}{\delta} \cdot \mathbbm{1}_{|x - \mu| > \frac{m_2}{\sqrt{\delta}}} \right]
        \\&= \frac{1}{\sigma^2}\log^2\frac{1}{\delta}\Pr\left[|x - \mu| > \frac{m_2}{\sqrt{\delta}} \right] 
        \\&= \frac{\delta}{\sigma^2}\log^2\frac{1}{\delta} \lesssim \frac{\sqrt{\delta}}{\sigma^2}
    \end{align*}
    The last line here uses the fact that for all $x$, $x\log^2 (1/x) \le 3\sqrt{x}$. Summing the two terms gives the desired result.
\end{proof}

Then, we show that neighborhoods where the magnitude of the score can be large are rare and can also be approximated by the zero function. This allows us to control the slope of the piecewise linear approximation in each piece.
\begin{lemma}
\label{lem:large_score_interval}
    Let $p$ be a distribution over $\R$. Let $p_\sigma = p * N(0, \sigma^2)$ have score $s_\sigma$. Let $\gamma \le \frac{\sigma}{2}$, and let $m(x) = \sup_{y \in [x-\gamma, x+\gamma]} s(x)$. Then,
    \begin{align*}
        \E\left[s(x)^2 \cdot \mathbbm{1}_{m(x) > \frac{\log\frac{1}{\delta}}{\sigma}} \right] \lesssim \frac{\sqrt{\delta}}{\sigma^2}
    \end{align*}
\end{lemma}
\begin{proof}
    \begin{align*}
        \E\left[m(x)^2 \cdot \mathbbm{1}_{m(x) > \frac{\log\frac{1}{\delta}}{\sigma}} \right] 
        &\le \sqrt{\E\left[m(x)^4 \right]\cdot \E\left[\mathbbm{1}_{m(x) > \frac{\log\frac{1}{\delta}}{\sigma}} \right]} && \text{ by Cauchy-Schwarz}
        \\&\le \sqrt{\E\left[\left(\sup_{y \in [x-\gamma, x+\gamma]}s(x)\right)^4 \right]\cdot \Pr\left[m(x) > \frac{\log\frac{1}{\delta}}{\sigma} \right]}
        \\&\lesssim \sqrt{\frac{1}{\sigma^4}\cdot \Pr\left[m(x) > \frac{\log\frac{1}{\delta}}{\sigma} \right]} && \text{ by Lemma \ref{cor:score_interval_moment_bound_simplified}}
        \\&\le \frac{\sqrt{\delta}}{\sigma^2}&&\text{ by Lemma \ref{lem:max_score_interval_bounded}}
    \end{align*}
\end{proof}
We put these lemmas together to show that a piecewise linear function with a bounded number of pieces and bounded slope in each piece is a good approximation to the smoothed score.

\begin{lemma}
\label{lem:score_linear_approximation}
    Let $p$ be a distribution over $\R$ with mean $\mu$, and let $p_\sigma = p* N(0, \sigma^2)$ have score $s_\sigma$ and second moment $m_2^2$. Then, for any $\alpha \le 1/4$ there exists a function $l : \R \to \R$ that satisfies
    \begin{enumerate}
        \item $l$ is piecewise linear with at most $\Theta(\frac{m_2}{\sigma\kappa^{3/2}})$ pieces,
        \item if $x$ is a transition point between two pieces, then $\abs{x-\mu} \le \frac{m_2}{\kappa}$
        \item the slope of each piece is bounded by $ \Theta\left(\frac{\log \frac{1}{\kappa}}{\sigma^2\sqrt{\kappa}}\right)$,
        \item $\abs{l} \lesssim \frac{1}{\sigma}\log \frac{1}{\kappa}$
        \item \[\E_{x \sim p}[(l(x) - s(x))^2] \lesssim \frac{\kappa}{\sigma^2}\]
    \end{enumerate}
\end{lemma}
\begin{proof}
    First, we partition the real line into $S_i = [i\gamma, (i+1)\gamma)$ for all $i \in \Z$, where $\gamma < \sigma/2$. Define the function $l_1:\R\to\R$ so that if $S_i \ni x$, then \begin{equation}
    \label{eq:l1}
        l_1(x) = \frac{((i+1)\gamma-x)s(i\gamma) + (x-i\gamma) s((i+1)\gamma)}{\gamma}.
    \end{equation} As in Lemma \ref{lem:linear_approximation}, this is the linear interpolation between $s(i\gamma)$ and $s((i+1)\gamma)$ on the interval $[i\gamma, (i+1)\gamma)$. By Lemma \ref{lem:linear_approximation}, when $\gamma < \sigma/2$, we have
    \[\E\left[(s(x) - l_1(x))^2 \right] \lesssim \frac{\gamma^2}{\sigma^4}\]
    Now, we define $l_2:\R\to\R$. This function uses the piecewise linear $l_1$ to create a linear approximation that has small slopes on all of the pieces. Define first a set of ``good" sets 
    \[G = \left\{S_i : \sup_{y \in S_i} s(x) \le \frac{1}{\sigma}\log\frac{1}{\delta}\right\}.\] These are the intervals on which the score is always bounded. Further, define two helper maps $L(x)$ and $U(x)$:
    \begin{align*}
        L(x) &= \text{the largest } i \text{ such that } i\gamma < x, S_{i-1} \in G
        \\R(x) &= \text{the smallest } i \text{ such that } i\gamma \ge x, S_i \in G
    \end{align*}
    These represent the nearest endpoint of a ``good" interval to the left and right, respectively. 
    We then interpolate linearly between $s(\gamma L(x))$ and $s(\gamma R(x))$ to evaluate $l_2(x)$. That is, 
    \begin{equation}
    \label{eq:l2}
        l_2(x) = \frac{(\gamma R(x)-x) s(\gamma L(x)) + (x - \gamma L(x)) s(\gamma R(x))}{\gamma (R(x) - L(x))}
    \end{equation}
    Note that by assumption, we have that $\abs{s(\gamma R(x))}, \abs{s(\gamma L(x))} \le \frac{1}{\sigma} \log\frac{1}{\delta}$, and so $\abs{l_2(x)} \le \frac{1}{\sigma} \log\frac{1}{\delta}$. 
    We now analyze the error of $l_2$ against $s$. First, we note that on the sets outside $G$, the error is bounded, using Lemma  \ref{lem:large_score_interval}:
    \begin{align*}
        \sum_{S_i \not\in G}\E\left[(s(x) - l_2(x))^2\mathbbm{1}_{x \in S_i}\right] &\le 2\sum_{S_i \not\in G}\left(\E\left[s(x)^2\mathbbm{1}_{x \in S_i}\right]+\E\left[l_2(x)^2\mathbbm{1}_{x \in S_i}\right]\right)
        \\&\lesssim \frac{\sqrt{\delta}}{\sigma^2} + \sum_{S_i \not\in G}\E\left[\frac{1}{\sigma^2}\log^2\frac{1}{\delta}\mathbbm{1}_{x \in S_i}\right] && \text{ by Lemma \ref{lem:large_score_interval}}
        \\&= \frac{\sqrt{\delta}}{\sigma^2} + \frac{1}{\sigma^2}\log^2\frac{1}{\delta}\Pr\left[x \not \in G\right]
        \\&\le\frac{\sqrt{\delta}}{\sigma^2} + \frac{1}{\sigma^2}\log^2\frac{1}{\delta}\Pr\left[\sup_{y \in [x-\gamma, x + \gamma]} s(x) \ge \frac{1}{\sigma}\log\frac{1}{\delta}\right]
        \\&\le\frac{\sqrt{\delta}}{\sigma^2} + \frac{\delta}{\sigma^2}\log^2\frac{1}{\delta} && \text { by Lemma \ref{lem:max_score_interval_bounded}}
    \end{align*}
    Further, if $x$ is in a ``good" interval, then $L(x), R(x)$ are simply the left and right endpoints of the interval that $x$ is in. This means that $l_2(x) = l_1(x)$. So,
    \begin{align*}
        \sum_{S_i \in G}\E\left[(s(x) - l_2(x))^2\mathbbm{1}_{x \in S_i}\right] &= \sum_{S_i \in G}\E\left[(s(x) - l_1(x))^2\mathbbm{1}_{x \in S_i}\right]
        \\&\le \sum_i\E\left[(s(x) - l_1(x))^2\mathbbm{1}_{x \in S_i}\right] \lesssim \frac{\sqrt{\delta}}{\sigma^2}
    \end{align*}
    Putting these two together, we get that 
    \begin{align*}
        \E\left[(l_2(x) - s(x))^2\right] \lesssim \frac{\sqrt{\delta}}{\sigma^2} + \frac{\gamma^2}{\sigma^4} + \frac{\delta}{\sigma^2}\log^2\frac{1}{\delta} 
    \end{align*}
    Now, define $l_3:\R\to\R$ as follows:
    \begin{equation}
    \label{eq:l3}
        l_3(x) = 
        \begin{cases}
            l_2(x) & \abs{x-\mu} \le \frac{m_2}{\sqrt{\delta}}
            \\ l_2\left(\mu-\frac{m_2}{\sqrt{\delta}}\right) & x < \mu - \frac{m_2}
            {\sqrt{\delta}}
            \\ l_2\left(\mu+\frac{m_2}{\sqrt{\delta}}\right) & x > \mu - \frac{m_2}{\sqrt{\delta}}
    \end{cases}
    \end{equation}
    This takes our previous approximation $l_2$ and holds it constant on values of $x$ far away from the mean.
    
    Let $B$ be the integers $i$ such that $x \in S_i \implies \abs{x-\mu} \ge m_2/\sqrt{\delta}$. In other words, the set $B$ enumerates the intervals on which $l_2 \neq l_1$, and equivalently, $l_2 = 0$. Note that since $\abs{l_3(x)} \le \frac{1}{\sigma} \log \frac{1}{\delta}$, we have in particular that $\abs{l_3\left(\mu \pm \frac{m_2}{\sqrt{\delta}}\right)} \le \frac{1}{\sigma} \log \frac{1}{\delta}.$ Therefore, for some $\abs{\varphi} \le \frac{1}{\sigma} \log \frac{1}{\delta}$, we have 
    \begin{align*}
        \E\left[(s(x) - l_3(x))^2 \right] &= \sum_i \E\left[(s(x) - l_3(x))^2 \mathbbm{1}_{x \in S_i}\right]
        \\&= \sum_{i \in B}\E\left[(s(x) - l_3(x))^2 \mathbbm{1}_{x \in S_i}\right] + \sum_{i \not\in B}\E\left[(s(x) - l_3(x))^2 \mathbbm{1}_{x \in S_i}\right]
        \\&=\sum_{i \in B}\E\left[(s(x)-\varphi)^2 \mathbbm{1}_{|x - \mu|\ge m_2/\sqrt{\delta} }\right] + \sum_{i}\E\left[(s(x) - l_2(x))^2 \mathbbm{1}_{x \in S_i}\right]
        \\&\lesssim \frac{\sqrt{\delta}}{\sigma^2} + \frac{\gamma^2}{\sigma^4}
    \end{align*}
    where this last line uses Lemma \ref{lem:large_x_score}.
    
    Finally, each piece of $l_3$ has slope at most $\Theta\left(\frac{\log \frac{1}{\delta}}{\gamma\sigma}\right)$ since the endpoints of each interval are bounded in magnitude by $\frac{1}{\sigma}\log\frac{1}{\delta}$ and each interval is at least $\gamma$ in width. Also, we can see that $l_3$ has at most as many pieces as $l_2$, which has $\Theta\left(\frac{m_2}{\gamma\sqrt{\delta}}\right)$ pieces, with each endpoint being within $m_2/\sqrt{\delta}$ of the mean. 

    So, we take $l$ to be $l_3$ with $\delta = \kappa^2$, and $\gamma = \sigma\sqrt{\kappa}$. Note that when $\kappa < 1/4$, we have $\gamma < \sigma/2$. Plugging these in, and using the fact that $x\log^2(1/x) \le 3\sqrt{x}$, we get that the number of pieces is $\Theta\left(\frac{m_2}{\sigma\kappa^{3/2}}\right)$, the slope of each piece is bounded by $\Theta\left(\frac{\log \frac{1}{\kappa^2}}{\sigma^2\sqrt{\kappa}}\right)$, the function itself is always bounded by $\frac{1}{\sigma}\log \frac{1}{\kappa^2}$, and $\E_{x \sim p}[\Vert l(x) - s(x)\Vert_2^2] \le \frac{\kappa}{\sigma^2}$. 
\end{proof}
Finally, we show that if we have a product distribution over $d$ dimensions, we can simply use the product of the one dimensional linear approximations along each coordinate to give a good approximation for the full score.
\begin{lemma}
    \label{lem:prod_approx}
    Let $p$ be a product distribution over $\R^d$, such that $p(x) = \prod_{i = 1}^d p_i(x_i).$ Let $s:\R^d \to \R^d$ be the score of $p$ and let $s_i: \R\to\R$ be the score of $p_i$. 
    If $l_i : \R \to \R$ is an approximation to $s_i$ such that 
    \[\E_{x_i \sim p_i}\left[( l_i(x_i) - s_i(x_i))^2\right]\le \varepsilon/d,\] then the function $l: \R^d \to \R^d$ defined as $l(x) = (l_i(x_i))$ satisfies
    \[\E_{x \sim p}\left[\Vert l(x) - s(x)\Vert_2^2 \right] \le \varepsilon\]
\end{lemma}
\begin{proof}
    We have 
    \[s(x)_i = \left(\grad\log p(x)\right)_i = \frac{\partial}{\partial x_i} \log \prod_{i=1}^d p_i(x_i) = \frac{\partial}{\partial x_i}\sum_{i=1}^d \log p_i(x_i) = \frac{\partial}{\partial x_i}\log p_i(x_i) = s_i(x_i).\]

    Therefore, 
    \begin{align*}
        \E_{x\sim p}\left[\Vert l(x) - s(x)\Vert_2^2 \right] 
            &= \E_{x\sim p}\left[\sum_{i = 1}^d \Vert l_i(x_i) - s_i(x_i)\Vert_2^2 \right]\\
            &= \sum_{i = 1}^d \E_{x_i\sim p_i}\left[\Vert l_i(x_i) - s_i(x_i)\Vert_2^2 \right]\le d\cdot \varepsilon/d = \varepsilon
    \end{align*}
\end{proof}
%
%

%% file: small_noise_level_lemmas.tex
\subsection{Small noise level -- Score of vertex distribution close to full score in vertex orthant}

\begin{lemma}[Density $g_s(x)$ is close to $g(x)$ for $s \in \{\pm 1\}^{d}$ closest to $x$]
\label{lem:density_of_g_x(y)_close_to_g(y)_for_appropriate_x}
Let $d' = O(d)$. Consider $g_s$ and $g$ as in Definition~\ref{def:unconditional_distribution}. We have that for $x \in \{\pm 1\}^{d}$ such that $s$ is closest to $x_{1, \dots, d}$ among points in $\{\pm 1\}^d$, for the $\sigma$-smoothed versions $h_s = g_s * \mathcal N(0, \sigma^2 I_{d+d'})$ of $g_s$ and $h = g * \mathcal N(0, \sigma^2 I_{d+d'})$ of $g$, for $\frac{R^2}{1+\sigma^2} > C \log d$ for sufficiently large constant $C$,
\begin{align*}
    \left|\frac{1}{2^d} h_s\left(x \right) - h(x) \right| \lesssim \frac{1}{2^d} \cdot e^{-\frac{R^2}{4(1+\sigma^2)}} 
\end{align*}
\end{lemma}
\begin{proof}
    We have that there are $\binom{d}{k}$ vectors $z \in \{\pm 1\}^d$ such that $\|R \cdot z - x_{1, \dots, d}\|^2 \ge k R^2$. For such a $z$,
    \begin{align*}
        h_z(y) \lesssim e^{-\frac{k R^2}{2(1+\sigma^2)}}
    \end{align*}
    So,
    \begin{align*}
        \left|\frac{1}{2^d} h_s\left(x \right) - h(x) \right| = \left|\frac{1}{2^d} \sum_{r \neq s}h_r(x) \right| \lesssim \left|\frac{1}{2^d}\sum_{k = 1}^d d^k e^{\frac{-k R^2}{2(1+\sigma^2)}} \right| \lesssim  \frac{1}{2^d} \cdot e^{-\frac{R^2}{4(1+\sigma^2)}}
    \end{align*}
    since $\frac{R^2}{1+\sigma^2} > C \log d$.
\end{proof}

\begin{lemma}[Gradient of density $g_x(y)$ is close to $g(y)$ for $x \in \{\pm 1\}^d$ closest to $y$]
\label{lem:grad_density_g_x(y)_close_to_g(y)_for_appropriate_x}
    Let $d' = O(d)$ and consider $g_s$ and $g$ as in Definition~\ref{def:unconditional_distribution}, and $x \in \R^{d+d'}$. We have that for $s \in \{\pm 1\}^d$ such that $s$ is closest to $x_{1, \dots, d}$ among points in $\{\pm 1\}^d$, for $\sigma \ge  \tau$, $\tau = \frac{1}{d^C}$ and $\eps > \frac{1}{\poly(d)}$, for the $\sigma$-smoothed versions $h_s = g_s * \mathcal N(0, \sigma^2 I_{d+d'})$ of $g_s$ and $h = g * \mathcal N(0, \sigma^2 I_{d+d'})$ of $g$, for $\frac{R^2}{1+\sigma^2} > C\log d$ for sufficiently large constant $C$,
    \begin{align*}
        \left\|\frac{1}{2^d} \grad h_s(x) - \grad h(x)\right\|^2 \lesssim \frac{1}{2^d} \cdot e^{-\frac{R^2}{16(1+\sigma^2)}}
    \end{align*}
\end{lemma}
\begin{proof}
    We will let $\wt h_{s, i} = \wt g_{s, i} * \mathcal N(0, \sigma^2)$, where $\wt g_{s, i}$ is defined in Definition~\ref{def:unconditional_distribution}. So, $h_s(x) = \prod_{i=1}^{d+d'} \wt h_{s, i}(x_i)$.
    We have that there are $\binom{d}{k}$ vectors $z \in \{\pm 1\}^d$ such that $\|R \cdot z - x_{1, \dots, d}\|^2 \ge k R^2$. So, for $i \in [d]$, for such a $z$,
    \begin{align*}
        |\left(\grad h_z(x) \right)_i| \lesssim e^{-\frac{k R^2}{4(1+\sigma^2)}}
    \end{align*}
    On the other hand, for $i > d$, by Lemma~\ref{lem:smoothed_discrete_univariate}, since $\sigma > \eps^2$ and $\eps > \frac{1}{\poly(d)}$,
    \begin{align*}
        \left|\wt h'_{z,i}(x_i) - w_{\sqrt{\sigma^2+1}}'(x_i)\right| &\lesssim e^{-\frac{\sigma^2}{2 \eps^2(1+\sigma^2)}} + \sum_{j > 0} e^{-\frac{j^2 \sigma^2}{2 \eps^2(1+\sigma^2)} + \log \frac{j}{\eps(1+\sigma^2)}}\\
        &\le e^{-\frac{\tau^2}{2 \eps^2 (1 + \tau^2)}} + \sum_{j > 0} e^{-\frac{j^2\tau^2}{2\eps^2(1+\tau^2)} + \log \frac{j}{\eps(1+\tau^2)}}\\
        &\lesssim \eps \sqrt{1 + \frac{1}{\tau^2}}\\
    \end{align*}
    So, we have that for $z \in \{\pm 1\}^{d}$ such that $\|R \cdot z - x_{1, \dots, d}\|^2 > k R^2$, since $\eps > \frac{1}{\poly(d)}$, $\tau = \frac{1}{\poly(d)}$ and $\frac{R^2}{1+\sigma^2} > C \log d$,
    \begin{align*}
        \left|\left(\grad h_z(x) \right)_i \right| \lesssim \eps \sqrt{1 + \frac{1}{\tau^2}} \cdot e^{-\frac{k R^2}{2(1+\sigma^2)}} \lesssim e^{-\frac{k R^2}{4(1+\sigma^2)}}
    \end{align*}
    So, finally, for such $z$,
    \begin{align*}
        \|\grad h_z(x)\|^2 \lesssim e^{-\frac{k R^2}{8(1+\sigma^2)}}
    \end{align*}
    Thus,
    \begin{align*}
        \left\| \frac{1}{2^d}\grad h_s(x) - \grad h(x)\right\|^2 = \left\|\frac{1}{2^d} \sum_{r \neq s} \grad h_r(x) \right\|^2\lesssim \frac{1}{2^d} \sum_{k=1}^d d^k e^{-\frac{k R^2}{8(1+\sigma^2)}} \lesssim \frac{1}{2^d} \cdot e^{-\frac{R^2}{16(1+\sigma^2)}}
    \end{align*}
\end{proof}
\begin{lemma}[Score of mixture close to score of closest (discretized) Gaussian]
    \label{lem:score_of_mixture_close_to_score_of_closest_conditionally}
    Let $d = O(d')$, and consider $g_s, g$ as in Definition~\ref{def:unconditional_distribution} for any $s \in \{\pm 1\}^d$, with $\frac{R^2}{1+\sigma^2} > C \log d$ for sufficiently large constant $C$. Let $S \subset \R^d$ be the orthant containing $s$. Let $\sigma \ge \tau$ for $\tau = \frac{1}{\poly(d)}$, and let $\eps > \frac{1}{\poly(d)}$. We have that, for the $\sigma$-smoothed scores $s_{\sigma, s}$ of $g_s$ and $s_\sigma$ of $g$,
    \begin{align*}
        \E_{x \sim h} \left[ \|s_{\sigma, s}(x) - s_\sigma(x)\|^2 \Big | 1_{x_{1, \dots, d} \in S} \right] \lesssim e^{-\Omega\left(\frac{R^2}{1+\sigma^2} \right)}
    \end{align*}
    where $h$ is the $\sigma$-smoothed version of $g$, given by $h = g * \mathcal N(0, \sigma^2 I_{d+d'})$.
\end{lemma}
\begin{proof}
    Let $h_s$ be the $\sigma$-smoothed version of $g_s$, given by $h_s = g_s * \mathcal N(0, \sigma^2 I_{d+d'})$. Let $\wt s \in \R^{d+d'}$ be such that the first $d$ coordinates are given by $s$, and the remaining $d'$ coordinates are $0$. We have
    \begin{align*}
        \E_{x \sim h}\left[ \|s_{\sigma, s}(x) - s_\sigma(x)\|^2 \Big | 1_{x_{1, \dots, d} \in S}\right] &= \E_{x \sim h}\left[ \|s_{\sigma, s}(x) - s_\sigma(x)\|^2 \cdot 1_{\|x - \wt s\| \le R/10} \Big | 1_{x_{1, \dots, d} \in S}\right]\\
        &+ \E_{x \sim h}\left[ \|s_{\sigma, s}(x) - s_\sigma(x)\|^2 \cdot 1_{\|x - \wt s\| > R/10} \Big | 1_{x_{1, \dots, d} \in S }\right]
    \end{align*}
    Note that when $\|x - \wt s\| \le R/10$, by Lemma~\ref{lem:smoothed_discrete_univariate}, $h_s(x) \gtrsim e^{-\frac{R^2}{64(1+\sigma^2)}}$ since $\sigma \ge \tau$ for $\tau = \frac{1}{\poly(d)}$, $\eps > \frac{1}{\poly(d)}$ and $\frac{R^2}{1+\sigma^2} > C \log d$. So, by Lemmas~\ref{lem:density_of_g_x(y)_close_to_g(y)_for_appropriate_x}~and~\ref{lem:grad_density_g_x(y)_close_to_g(y)_for_appropriate_x}, $h(x) = \frac{1}{2^d} h_s(x)\left(1 + O\left(e^{-\frac{R^2}{8(1+\sigma^2)}} \right) \right)$, and $\|\grad h(x) - \frac{1}{2^d} \grad h_s(x)\|^2 \lesssim \frac{1}{2^d} e^{-\frac{R^2}{16(1+\sigma^2)}}$. Also note that by Lemma~\ref{lem:density_of_g_x(y)_close_to_g(y)_for_appropriate_x}, $h(x | x_{1, \dots, d} \in S) \le h_s(x) + O(e^{-\frac{R^2}{8(1+\sigma^2)}}) \le h_s(x) \cdot \left(1 + O\left(e^{-\frac{R^2}{32(1+\sigma^2)}} \right) \right)$. So, for the first term,
    \begin{align*}
        &\E_{x \sim h}\left[ \|s_{\sigma, s}(x) - s_\sigma(x)\|^2 \cdot 1_{\|x - \wt s\| \le \frac{R}{10}}\Big | 1_{x_{1, \dots, d} \in S }\right]\\
        &= \E_{x \sim h}\left[\left\|\frac{\frac{1}{2^d} \grad h_s(x)}{\frac{1}{2^d} h_s(x)} - \frac{\grad h(x)}{h(x)} \right\|^2 \cdot 1_{\|x - \wt s\| \le R/10} \Big | 1_{x_{1, \dots, d} \in S } \right]\\
        &\lesssim \E_{x \sim h}\left[\frac{\frac{1}{2^d} e^{-\frac{R^2}{16(1+\sigma^2)}} + e^{-\frac{R^2}{8(1+\sigma^2)}} \cdot \frac{1}{2^d} \cdot \|\grad h_s(x)\|^2}{\frac{1}{2^d} h_s(x)^2} \cdot 1_{\|x - \wt s\| \le R/10} \Big | 1_{x_{1,\dots, d} \in S } \right]\\
        &\lesssim e^{-\frac{R^2}{32(1+\sigma^2)}} + e^{-\frac{R^2}{8(1+\sigma^2)}} \cdot \E_{x \sim h_s}\left[ \frac{\|\grad h_s(x)\|^2}{h_s(x)^2} \right]\\
        &\lesssim e^{-\frac{R^2}{32(1+\sigma^2)}} + \frac{d e^{-\frac{R^2}{8 (1+\sigma^2)}}}{\sigma^2}\\
        &\lesssim e^{-\frac{R^2}{64(1+\sigma^2)}}
    \end{align*}
    since $\sigma \ge \frac{1}{\poly(d)}$, $\eps > \frac{1}{\poly(d)}$ and $\frac{R^2}{1+\sigma^2} > C \log d$.
    
    For the second term, by Cauchy-Schwarz,
    \begin{align*}
        &\E_{x \sim h}\left[\|s_{\sigma, s}(x) - s_\sigma(x)\|^2 \cdot 1_{\|x - \wt s\| > \frac{R}{10}} \Big | 1_{x_{1, \dots, d} \in S}\right] \\
        &\lesssim \sqrt{\left(\E_{x \sim h}\left[\|s_{\sigma, s}(x)\|^4 + \|s_\sigma(x)\|^4 \Big | 1_{x_{1, \dots, d} \in S} \right]\right)\cdot \E\left[1_{ \cap \|x - \wt s\| > R/10} \Big | 1_{x_{1, \dots, d} \in S} \right]}\\
        &\lesssim \sqrt{\frac{R^4}{\sigma^4} + \frac{1}{\sigma^4} \E\left[\|x\|^4 \Big | 1_{x_{1, \dots, d} \in S} \right]}\cdot e^{-\Omega\left(\frac{R^2}{1 + \sigma^2}\right)}\\
        &= \frac{1}{\sigma^2} \sqrt{R^4 + \E_{s \sim \{\pm 1\}^d}\left[\E\left[\|x\|^4 \Big | 1_{x_{1, \dots, d} \in S}, x \sim g_s\right]\right]} e^{-\Omega\left(\frac{R^2}{1+\sigma^2} \right)}\\
        &\lesssim \frac{R^2}{\sigma^2} \cdot e^{-\Omega\left(\frac{R^2}{1+\sigma^2} \right)}\\
        &\lesssim e^{\Omega\left(\frac{R^2}{1+\sigma^2} \right)}
    \end{align*}
    So, we have the claim.
\end{proof}

%% file: relu_approx_of_product.tex
\subsection{ReLU Network approximation of $\sigma$-smoothed Scores of Product Distributions}
%
%
%

Once we have this, we also need to go from being close to mixture of Gaussians to being close to mixture of discretized Gaussians.

 \begin{lemma}
    \label{lem:ReLU_linpiecewise}
        Let $f : \R \to \R$ be a continuous piecewise linear function with $D$ segments. Then, $f$ can be represented by a ReLU network with $O(D)$ parameters. If each segment's slope, each transition point, and the values of the transition points are at most $\beta$ in absolute value, each parameter of the network is bounded by $O(\beta)$ in absolute value.
    \end{lemma}
    \begin{proof}
        Since $f$ is piecewise linear, we can define $f$ as follows: there exists $-\infty = \gamma_0 < \gamma_1 < \gamma_2 < \dots < \gamma_{D - 1} = \gamma_D = +\infty$ such that \[
            f(x) = \left \{
            \begin{array}{cl}
                a_1 x + b_1, & x \le \gamma_1 \\
               a_2 x + b_2, & \gamma_1 < x \le \gamma_2 \\
               \vdots & \\
               a_D x + b_D, & \gamma_{D - 1} < x,
            \end{array} \right. 
        \]
        where $a_k \gamma_k + b_k = a_{k+1} \gamma_k + b_{k+1}$ for each $k \in [D-1]$. Now we will show that $f(x)$ equals $g(x)$ defined below: 
        \begin{equation*}
         g(x) := a_1 x + b_1 + \sum_{i=2}^D \ReLU ((a_i - a_{i-1}) (x - \gamma_{i-1})).           
        \end{equation*}
        We observe that for $\gamma_{k-1} < x \le \gamma_{k}$, \[
            g(x) = a_1 x + b_1 + \sum_{i=2}^{k}  (a_i - a_{i-1}) (x - \gamma_{i-1}) 
            = a_k x - \sum_{i=2}^{k}  (a_i - a_{i-1}) \gamma_{i-1}.
        \]
        Then, when $k > 1$, for $\gamma_{k-1} < x \le \gamma_{k}$, we have \begin{align*}
            g(x) &= a_kx - \sum_{i=2}^k  (a_i - a_{i-1}) \gamma_{i-1} \\ 
            &= \left({a_{k-1}\gamma_{k-1} - \sum_{i=2}^{k-1}  (a_i - a_{i-1}) \gamma_{i-1}}\right) + a_k x - a_{k-1}\gamma_{k-1} - (a_k - a_{k-1}) \gamma_{k-1} \\
            &= g(\gamma_{k-1}) + a_k x - a_k\gamma_{k-1}.
        \end{align*}
       Using these observations, we can inductively show that for each $k \in [D]$, $g(x) = f(x)$ holds for $\gamma_{k-1} < x \le \gamma_k$. 
       For $x \le \gamma_1$, \[
        g(x) = a_1 x + b_1 = f(x).
       \]
       Assuming for $\gamma_{k-2} < x \le \gamma_{k-1}$, $g(x) = f(x)$. Then $g(\gamma_{k-1}) = f(\gamma_{k-1}) = a_k \gamma_{k-1} + b_k$. Therefore, for $\gamma_{k-1} < x \le \gamma_{k}$, we have \[
            g(x) = g(\gamma_{k-1}) + a_k x - a_k\gamma_{k-1} = a_k x + b_k = f(x).
       \]
       This proves that $g(x) = f(x)$ for $x \in \R$ and we only need to design neural network to represent $g$.
       By employing one neuron for $a_1 x + b_1$ and $D - 1$ neurons for $\ReLU ((a_i - a_{i-1}) (x - \gamma_{i-1}))$, and aggregating their outputs, we obtain the function $g$. There are $O(D)$ parameters in total, and each parameter is bounded by $O(\beta)$ in absolute value.
    \end{proof}

    \begin{lemma}
    \label{lem:prod_relu_param}
    Let $f_1, \dots, f_k$ be functions mapping $\R$ to $\R$. Suppose each $f_i$ can be represented by a neural network with $p$ parameters bounded by $\beta$ in absolute value. Then, function $g : \R^k \to \R^k$ defined by \[
        g(x_1, \dots, x_k) := (f_1(x_1), \dots, f_k(x_k))
    \] can be represented by a neural network with $O(pk)$ parameters bounded by $\beta$ in absolute value.
    \end{lemma}
    \begin{proof}
        We just need to deal with each coordinate separately and use the neural network representation for each $f_i$. We just need to concatenate each result of $f_i$ together as the final output.
    \end{proof}

\begin{lemma}[ReLU network implementing the score of a one-dimensional $\sigma$-smoothed distribution]
\label{lem:one_dimensional_score_relu}
    Let $p$ be a distribution over $\R$ with mean $\mu$, and let $p_\sigma = p * \mathcal N(0, \sigma^2)$ have variance $m_2^2$ and score $s_\sigma$. There exists a constant-depth ReLU network $f :  \R \to \R$ with $O(\frac{m_2}{\gamma^3 \sigma^4})$ parameters with absolute values bounded by $O(\frac{m_2}{\sigma^2 \gamma^2} + \frac{\log \frac{1}{\gamma}}{\sigma^3\gamma} + \abs{\mu})$ such that
    \begin{align*}
        \E_{x \sim p_\sigma}\left[ \| s_\sigma(x) - f(x) \|^2\right] \lesssim \gamma^2
    \end{align*}
    and
    \begin{align*}
        |f(x)| \lesssim \frac{1}{\sigma} \log \frac{1}{\sigma \gamma}
    \end{align*}
\end{lemma}
\begin{proof}
    By Lemma~\ref{lem:score_linear_approximation}, there exists a continuous piecewise approximation of $p$ with $O(\frac{m_2}{\sigma^3 \gamma^4})$ pieces with each segment's slope, each transition point, and function value all bounded in $O(\frac{m_2}{\sigma^2 \gamma^2} + \frac{1}{\sigma^3 \gamma} \log\frac{1}{\sigma\gamma} + \frac{1}{\sigma}\log \frac{1}{\sigma \gamma} + \abs{\mu})$. Taking this into    \ref{lem:ReLU_linpiecewise} and we have the bound.    
\end{proof}
\begin{lemma}
    \label{lem:product_score_relu}
    Let $p$ be a product distribution over $\R^d$ such that $p(x) = \prod_{i = 1}^d p_i(x_i)$, and let $p_\sigma = p * \mathcal N(0, \sigma^2 I_d)$ have score $s_\sigma$. Assume $p_\sigma$ has mean $\mu$ and variance $m_2^2 = \E_p[\|x - \mu\|_2^2]$. Then, there exists a constant-depth ReLU network $f :  \R^d \to \R^d$ with $O(\frac{dm_2}{\gamma^3 \sigma^4})$ parameters with absolute values bounded by $O(\frac{dm_2}{\sigma^2 \gamma^2} + \frac{\sqrt{d}}{\sigma^3 \gamma} \log\frac{d}{\sigma\gamma} + \norm{\mu}_1)$ such that
    \begin{align*}
        \E_{x \sim p_\sigma}\left[ \| s_\sigma(x) - f(x) \|^2\right] \lesssim \gamma^2.
    \end{align*} 
    and
    \begin{align*}
        |f(x)_i| \lesssim \frac{1}{\sigma} \log \frac{1}{\sigma \gamma}
    \end{align*}
\end{lemma}
\begin{proof}
    Consider distribution $p_i : \R \to \R$ and its $\sigma$-smoothed version $p_{i\sigma} = p_i * \mathcal N(0, \sigma^2)$. Let $\mu_i$ and $m_{2i}$ be the mean and the variance of $p_i$ respectively. Let $s_{\sigma i}$ be the $i$-th component of $s_\sigma$. Then,  Lemma~\ref{lem:one_dimensional_score_relu} shows that for each $i\in [d]$, there exists a constant-depth ReLU network $f_i :  \R \to \R$ with $O(\frac{m_{2i}}{\gamma^3 \sigma^4})$ parameters with absolute values bounded by $O(\frac{dm_2}{\sigma^2 \gamma^2} + \frac{\sqrt{d}}{\sigma^3 \gamma} \log\frac{d}{\sigma\gamma} + |\mu_i|)$ such that
    \begin{align*}
        \E_{x \sim p_{\sigma i}}\left[ \| s_{\sigma i}(x) - f_i(x) \|^2\right] \lesssim \frac{\gamma^2}{d}.
    \end{align*} 
    Then, we can use the product function $f = (f_1, \dots, f_d)$ as the approximation for $s_\sigma$. By Lemma~$\ref{lem:prod_approx}$,    
    \begin{align*}
        \E_{x \sim p_{\sigma}}\left[ \| s_\sigma(x) - f(x) \|^2\right] \lesssim \gamma^2.
    \end{align*} 
    Taking the fact that $\sum_{i\in [d]} |\mu_i| = \norm{\mu_1}$ and $\sum_{i \in [d]} m_{2i} \le dm_2$ into
    Lemma~\ref{lem:prod_relu_param}, and we prove the statement.
    
\end{proof}

%% file: relu_approx_small_smoothing.tex
\subsection{ReLU network for Score at Small smoothing level}

%

\begin{lemma}[Vertex Identifier Network]
    \label{lem:hypercube_identifier_relu}
    For any $0< \alpha < 1$, there exists a ReLU network $h: \R^d \to \R^d$ with $O(d/\alpha)$ parameters, constant depth, and weights bounded by $O(1/\alpha)$ such that
    \begin{itemize}
        \item If $|x_i| > \alpha$, for all $ i \in [d]$, then $h(x)_i = \frac{x_i}{|x_i|}$ for all $i \in [d]$.
    \end{itemize}
\end{lemma}
\begin{proof}
    Consider the one-dimensional function
    \begin{align*}
        g(y) = \begin{cases}
            -1, \quad y \le -\alpha\\
            \frac{y}{\alpha}, \quad -\alpha < y < \alpha\\
            1, \quad y \ge \alpha
        \end{cases}
    \end{align*}
    This is a piecewise linear function, where the derivative of each piece is bounded by $\frac{1}{\alpha}$, the value of the transition points are at most $\alpha$ in absolute value, and $|h|$ itself is bounded by $1$. Thus, by Lemma~\ref{lem:ReLU_linpiecewise}, we can represent the function $h(x) = \left(g(x_1), \dots, g(x_d) \right)$ using $O(d/\alpha)$ parameters, with each parameter's absolute value bounded by $O(1/\alpha)$. Moreover, clearly $h(x)_i = \frac{x_i}{|x_i|}$ for all $i \in [d]$ whenever $|x_i| \ge \frac{1}{C}$.
\end{proof}


\begin{lemma}[Switch Network]
    \label{lem:switch_relu}
    Consider any function $\switch: \R^{d+1} \to \R^d$ such that for $x \in \R^d$, $y \in \R$, with $|x_i| \le T$ for all $i \in [d]$,
    \begin{align*}
        \switch(x, y) =
        \begin{cases}
            x &\text{if $y = 1$}\\
            0 &\text{if $y = -1$} 
        \end{cases}
    \end{align*}
    $\switch$ can be implemented using a constant depth ReLU network with $O(d T)$ parameters, with each parameter's absolute value bounded by $O(T)$.
\end{lemma}

\begin{proof}
    Consider the ReLU network given by
    \begin{align*}
        \switch(x, y)_i = \ReLU((x_i - 2 T) + 2T \cdot y) - \ReLU((-x_i - 2 T) + 2 T \cdot y)
    \end{align*}
    It computes our claimed function. Moreover, it is constant-depth, the number of parameters is $O(d T)$, and each parameter is bounded by $O(T)$ in absolute value, as claimed.
\end{proof}

\begin{lemma}
    \label{lem:relu_small_smoothing}
    Let $d' = O(d)$. Given a constant-depth ReLU network representing a one-way function $f : \{-1, 1\}^{d} \to \{-1, 1\}^{d'}$ with $\poly(d)$ parameters, there is a constant-depth ReLU network $h : \R^{d+d'} \to \R^{d+d'}$ with $\poly\left(\frac{d}{\sigma \gamma}\right)$ parameters with each parameter bounded in absolute value by $\poly\left(\frac{d}{\sigma \gamma} \right)$ such that for the unconditional distribution $g$ defined in Definition~\ref{def:unconditional_distribution} with $\sigma$-smoothed version $g_\sigma$ and corresponding score $s_\sigma$, for $\tau = \frac{1}{d^C}$ and $\tau \le \sigma < \frac{R}{C \sqrt{\log d}}$ for sufficiently large constant $C$, and $R > C \log d$, $\eps > \frac{1}{\poly(d)}$, $\gamma > \frac{1}{d^{C/100}}$
    \begin{align*}
        \E_{x \sim g_\sigma} \left[\|s_\sigma(x) - h(x)\|^2 \right] \lesssim \gamma^2
    \end{align*}
\end{lemma}

\begin{proof}
    We will let our ReLU network $h$ be as follows. Let $r$ be the ReLU network from Lemma~\ref{lem:hypercube_identifier_relu} that identifies the closest hypercube vertex with any constant parameter $\alpha < 1$.
    
    For each $i \in [d]$, we will let $\wt h_i$ be the ReLU network that implements the approximation to the score of the one-dimensional distribution $w_{1}(x) * \mathcal N(0, \sigma^2)$ from Lemma~\ref{lem:one_dimensional_score_relu}. By the lemma, it satisfies
    \begin{align}
        \label{eq:for_i<d_wt_h_approximates_gaussian}
        \E_{x \sim w_{\sqrt{\sigma^2+1}}}\left[(\wt h_i(x) - \grad \log w_{\sqrt{\sigma^2+1}}(x))^2 \right] \lesssim \gamma^2
    \end{align}  
    For $i \in d + [d']$, we will let $\wt h_{i, 1}$ be the ReLU network that implements the approximation to the score $\wt s_{\sigma, i, -1}$ of $\wt g_{\sigma, i, -1} = (\frac{w_1 \cdot \comb_\eps}{\int w_1(x) \cdot \comb_\eps(x) dx}) * \mathcal N(0, \sigma^2)$, and we will let $\wt h_{i, -1}$ implement the approximation to the score of $\left(\frac{w_1(x) \cdot \comb_\eps \left(x - \eps/2 \right)}{\int w_1(x) \cdot \comb_\eps(x-\eps/2) dx} \right) * \mathcal N(0, \sigma^2)$, as given by Lemma~\ref{lem:one_dimensional_score_relu}. By the Lemma, for every $i \in d + [d']$ and $j \in \{\pm 1\}$, we have
    \begin{align}
    \label{eq:for_i>d_wt_h_approximates_comb_score}
        \E_{x \sim \wt g_{\sigma, i, j}}\left[(\wt h_{i, j}(x) - s_{\sigma, i, j}(x))^2 \right] \lesssim \gamma^2
    \end{align}
    
    Note that each $|\wt h_i| \le \frac{C}{\sigma} \log \frac{1}{\sigma \gamma}$ for $i \le d$, and $|\wt h_{i, \pm 1}| \le \frac{C}{\sigma} \log \frac{1}{\sigma \gamma}$ for $i > d$, for sufficiently large constant $C$.

    Now let $\switch$ be the ReLU network described in Lemma~\ref{lem:switch_relu} for $T = \frac{C}{\sigma} \log \frac{1}{\sigma \gamma}$. 

    Consider the network $ h : \R^{d+d'} \to \R^{d+d'}$ given by
    \begin{align*}
        h(x)_i =
        \begin{cases}
            \wt h_i(x_i - r(x)_i \cdot R) & \text{for $i \le d$}\\
            \switch(\wt h_{i, 1}(x_i), f(r(x))_{i-d}) + \switch(\wt h_{i, -1}(x_i), -f(r(x))_{i-d}) & \text{for $i > d$}\\
        \end{cases}
    \end{align*}

    Note that $h$ can be represented with $\poly\left(\frac{d}{\sigma \gamma} \right)$ parameters with absolute value of each parameter bounded in $\poly\left( \frac{d}{\sigma \gamma}\right)$. We will show that $h$ approximates $s_\sigma$ well in multiple steps.

    For $r(x) \in \{\pm 1\}^d$, consider the score $s_{\sigma, r(x)}$ of $g_{\sigma, r(x)}$, the $\sigma$-smoothed version of the distribution $g_{r(x)}$ centered at $\wt {r(x)} \in \R^{d+d'}$, as described in Definition~\ref{def:unconditional_distribution}, where $\wt {r(x)}$ has the first $d$ coordinates given by $r(x)$, and the remaining coordinates set to $0$.
    \paragraph{Whenever $r(x) = j \in \{\pm 1\}^{d}$, $h$ approximates $s_{\sigma, j}$ well over $g_{\sigma, j}$.}
    
    We will show that for fixed $j \in \{\pm 1\}^d$
    \begin{align*}
        \E_{x \sim g_{\sigma, j} }\left[\|s_{\sigma, j}(x) - h(x)\|^2 \cdot 1_{r(x) = j} \right] \lesssim d \gamma^2
    \end{align*}
    
     First, note that for $i \le d$, by \eqref{eq:for_i<d_wt_h_approximates_gaussian} and our definition of $h$,
     \begin{align*}
        \E_{x \sim w_{\sqrt{\sigma^2 + 1}}}\left[(h(x)_i - \grad \log w_{\sqrt{\sigma^2 + 1}}(x - R \cdot j))^2 \cdot 1_{r(x) = j} \right] \lesssim \gamma^2
     \end{align*}
    On the other hand, for $i >d$, by \eqref{eq:for_i>d_wt_h_approximates_comb_score} and our definition of $h$,
    \begin{align*}
        \E_{x\sim \wt g_{\sigma, i, f(j)_{i-d}}}\left[ (h(x)_i - s_{\sigma, i, f(j)_{i-d}})^2 \cdot 1_{r(x) = j} \right] \lesssim \gamma^2
    \end{align*}
    Since by Definition~\ref{def:unconditional_distribution}, for $j \in \{\pm 1\}^d$, $g_{\sigma, j}(x) =\prod_{i=1}^d w_{\sqrt{\sigma^2+1}}(x) \cdot\prod_{i=d+1}^{d+d'} \wt g_{\sigma, i, f(j)_{i-d}}(x)$, we have by Lemma~\ref{lem:prod_approx},
    \begin{align*}
        \E_{x \sim g_{\sigma, j}}\left[\|h(x) - s_{\sigma, j}(x) \|^2 \cdot 1_{r(x) = j} \right] \lesssim d \gamma^2 
    \end{align*}

    \paragraph{$h$ approximates $s_{\sigma, j}$ exponentially accurately over $g_\sigma$.}
    By Lemma~\ref{lem:density_of_g_x(y)_close_to_g(y)_for_appropriate_x}, we have that for $x$ such that $r(x) = j$,
    \begin{align*}
        \left|\frac{1}{2^d} g_{\sigma, j}(x) - g_\sigma(x)\right| \lesssim \frac{1}{2^d} \cdot e^{-\frac{R^2}{4(1+\sigma^2)}} \lesssim \frac{1}{2^d}
    \end{align*}
    for our choice of $R, \sigma$.
    
    So, we have
    \begin{align*}
        \E_{x \sim g_\sigma}\left[\|h(x) - s_{\sigma, j}(x)\|^2 \cdot 1_{r(x) = j} \right] \lesssim \frac{d \gamma^2}{2^d}
    \end{align*}

    \paragraph{$h$ approximates $s_\sigma$ well over $g_\sigma$ whenever $r(x) \in \{\pm 1\}^d$.}
        Summing the above over $j \in \{\pm 1\}^d$ gives
        \begin{align*}
            \E_{x \sim g_\sigma}\left[\|h(x) - s_{\sigma, r(x)}(x)\|^2 \cdot 1_{r(x) \in \{\pm 1\}^d} \right]\lesssim d \gamma^2
        \end{align*}
        Moreover, by Lemma~\ref{lem:score_of_mixture_close_to_score_of_closest_conditionally}, for $\Bar{r}(x) = y$ where $y \in \{\pm 1\}^d$ represents the orthant that $x \in \{\pm 1\}^d$ belongs to,
        \begin{align*}
            \E_{x \sim g_\sigma}\left[\|s_{\sigma, \Bar{r}(x)} - s_\sigma(x)\|^2 \right] \lesssim e^{-\Omega\left(\frac{R^2}{1+\sigma^2} \right)} \lesssim \frac{1}{d^{C^2/10}}
        \end{align*}
        So, by the above, we have that 
        \begin{align*}
            \E_{x \sim g_\sigma}\left[\|h(x) - s_\sigma(x)\|^2 \cdot 1_{r(x) \in \{\pm 1\}^d} \right]  \lesssim d \gamma^2 + \frac{1}{d^{C^2/10}}
        \end{align*}
    \paragraph{Contribution of $x$ such that $r(x) \not\in \{\pm 1\}^d$ is small.}
    By the definition of $r$, 
    \begin{align*}
        \Pr_{x \sim g_\sigma}\left[r(x) \not \in \{\pm 1\}^d \right] \lesssim d e^{-\frac{(R - \alpha)^2}{2}} \lesssim e^{-\frac{R^2}{4}}
    \end{align*}
    So, by Cauchy-Schwarz,
    \begin{align*}
        \E_{x \sim g_\sigma}\left[\|s_\sigma(x)\|^2 \cdot 1_{r(x) \not\in \{\pm 1\}^d} \right] &\le \sqrt{\E_{x \sim g_\sigma}\left[\| s_\sigma(x)\|^4 \right] \cdot \Pr_{x \sim g_\sigma}\left[r(x) \not \in \{\pm 1\}^d \right]}\\
        &\lesssim \frac{1}{\sigma^2} e^{-\frac{R^2}{4}}\\
        &\lesssim e^{-\frac{R^2}{8}}
    \end{align*}
    Similarly, since $|h(x)_i| \le \frac{C}{\sigma} \log \frac{1}{\sigma \gamma}$, we have
    \begin{align*}
        \E_{x \sim g_\sigma}\left[\|h(x)\|^2 \cdot 1_{r(x) \not \in \{\pm 1\}^d} \right] \lesssim \frac{1}{\sigma^2} \log^2 \frac{1}{\sigma \gamma} \cdot e^{-\frac{R^2}{4}} \lesssim e^{-\frac{R^2}{8}}
    \end{align*}
    Thus, we have 
    \begin{align*}
        \E_{x \sim g_\sigma}\left[ \|h(x) - s_\sigma(x)\|^2 \cdot 1_{r(x) \not \in \{\pm 1\}^d} \right] \lesssim e^{-\frac{R^2}{8}} \lesssim \frac{1}{d^{C^2/40}}
    \end{align*}
    \paragraph{Putting it together.}
    By the above, we have,
    \begin{align*}
        \E_{x \sim g_\sigma}\left[\|h(x) - s_\sigma(x)\|^2 \right] \lesssim d\gamma^2 + \frac{1}{d^{C^2/40}}
    \end{align*}
    Reparameterizing $\gamma$ and noting that $\gamma > \frac{1}{d^{C/100}}$ gives the claim.
\end{proof}

%% file: smoothing_discretized_gaussian.tex
\subsection{Smoothing a discretized Gaussian}

\begin{lemma}
    \label{lem:smoothed_discrete_univariate}
    For any $\phi$, let $g$ be the univariate discrete Gaussian with pdf
    \begin{align*}
        g(x) \propto w_1(x) \cdot \text{comb}_\eps(x - \phi)
    \end{align*}
    Consider the $\rho$-smoothed version of $g$, given by $g_\rho = g * w_\rho$. We have that
    \begin{align*}
        \left|g_\rho(x) - w_{\sqrt{\rho^2 + 1}}(x)\right| \lesssim e^{-\frac{\rho^2}{2 \eps^2(1+\rho^2)}} \cdot w_{\sqrt{\rho^2 + 1}}(x)
    \end{align*}
    and
    \begin{align*}
        \left| g_{\rho}'(x) - w'_{\sqrt{\rho^2 + 1}}(x)  \right| \lesssim  e^{-\frac{\rho^2}{2 \eps^2 (1 + \rho^2)}} \cdot |w'_{\sqrt{\rho^2 + 1}}(x)| + \sum_{j > 0} e^{-\frac{j^2 \rho^2}{2 \eps^2(1+\rho)^2}} \cdot \frac{j}{\eps(1+\rho^2)} \cdot w_{\sqrt{\rho^2+1}}(x)
    \end{align*}
\end{lemma}
\begin{proof}
    The Fourier transform of the $\comb_\eps$ distribution is given by $\frac{1}{\eps} \comb_{1/\eps}$. So, for the discrete Gaussian $g$, we have that its Fourier Transform is given by
    \begin{align*}
        \wh g(\xi) &= \left(\wh w_{1} * \left(\frac{e^{-i \xi \phi}}{\eps}\comb_{1/\eps}\right)\right)(\xi)\\
        &= \frac{1}{\eps} \cdot \sum_{j \in \Z} e^{-i \frac{j}{\eps} \phi} \cdot e^{-\frac{(\xi - \frac{j}{\eps})^2}{2}}
    \end{align*}
    Then, for $g_\rho$, the $\rho$-smoothed version of $g$, we have that its Fourier Transform is
    \begin{align*}
    \wh g_\rho(\xi) &= (\wh g \cdot \wh w_{1/\rho})(\xi) \\
    &=\frac{1}{\eps}  \sum_{j = -k}^{k} e^{-\frac{i j}{\eps} \phi} e^{-\frac{\xi^2 \rho^2}{2}} \cdot e^{-\frac{\left( \xi - \frac{j}{\eps}\right)^2}{2}}
\end{align*}
So, we have that, by the inverse Fourier transform,
\begin{align*}
    g_\rho(x) &= w_{\sqrt{\rho^2 + 1}}(x) + \frac{1}{2\pi} \int_{-\infty}^\infty e^{i \xi x} \sum_{j \neq 0} e^{-\frac{ij}{\eps} \phi} e^{-\frac{\xi^2 \rho^2}{2} - \frac{(\xi - \frac{j}{\eps})^2}{2}} d\xi\\
    &= w_{\sqrt{\rho^2 + 1}}(x) + \frac{1}{2 \pi} \int_{-\infty}^\infty e^{i \xi x} \cdot e^{-\frac{j^2}{2 \eps^2} + \frac{j^2}{2 \eps^2 (\rho^2 + 1)}}\sum_{j \neq 0} e^{-\frac{ij}{\eps} \phi} \cdot e^{-\frac{\rho^2 + 1}{2} \left(\xi - \frac{j}{\eps(\rho^2 + 1)} \right)^2}\\
    &= w_{\sqrt{\rho^2 +1}}(x) + \sum_{j \neq 0} e^{-\frac{ij}{\eps} \phi} \cdot e^{-\frac{j^2 \rho^2}{2 \eps^2(\rho^2 + 1)}} e^{i x \frac{j}{\eps (\rho^2 + 1)}} \cdot w_{\sqrt{\rho^2 + 1}}(x)
\end{align*}
so that 
\begin{align*}
    \left|g_\rho(x) - w_{\sqrt{\rho^2 +1}}(x)\right| \le \sum_{j \neq 0} e^{-\frac{j^2 \rho^2}{2 \eps^2 (\rho^2 + 1)}} w_{\sqrt{\rho^2 +1}}(x) 
\end{align*}
giving the first claim. For the second claim, note that the above gives
\begin{align*}
    g_\rho'(x) = w'_{\sqrt{\rho^2+1}}(x) + \sum_{j \neq 0} e^{-\frac{i j \phi}{\eps}} \cdot e^{-\frac{j^2 \rho^2}{2 \eps^2(\rho^2+1)}}  e^{\frac{ix j}{\eps(\rho^2+1)}}\cdot \left(\frac{i j}{\eps(\rho^2+1)} w_{\sqrt{\rho^2+1}}(x) + w'_{\sqrt{\rho^2+1}}(x) \right)
\end{align*}
So,
\begin{align*}
    \left|g_\rho'(x) - w'_{\sqrt{\rho^2+1}}(x) \right| \lesssim e^{-\frac{\rho^2}{2 \eps^2(1+\rho^2)}} \cdot \left|w'_{\sqrt{\rho^2+1}}(x)\right| + \sum_{j > 0} e^{-\frac{j^2 \rho^2}{2 \eps^2(1+\rho^2)}} \cdot \frac{j}{\eps(1+\rho^2)} w_{\sqrt{\rho^2+1}}(x)
\end{align*}


\end{proof}

%% file: large_noise_level_lemmas.tex
\subsection{Large Noise Level - Distribution is close to a mixture of Gaussians}

\begin{lemma}
    Let $C$ be a sufficiently large constant. Let $g^j(x) \propto \prod_{i=1}^d w_1(x_i) \cdot \prod_{i=d+1}^{d'} w_1(x_i) \cdot \text{comb}_\eps(x_i-\phi_{i,j})$ be the pdf of a distribution on $\R^{d+d'}$ with shifts $\phi_{i, j}$. Consider a mixture of discrete $d$-dimensional Gaussians, given by the pdf
    \begin{align*}
        h(x) = \sum_{j=1}^k \beta_j g^j(x - \mu_j)
    \end{align*}
    Let $h_{\rho}(x) = (h * w_{\rho})(x)$ be the smoothed version of $h$.
    Then, for the mixture of standard $(d+d')$-dimensional Gaussians given by
    \begin{align*}
        f_\rho(x) = \sum_{j=1}^k \beta_j w_{\sqrt{\rho^2 + 1}}(x - \mu_j)
    \end{align*}
    for $\frac{\rho^2}{\eps^2(1+\rho^2)} > C \log d$, we have that
    \begin{align*}
        \E_{x \sim h_\rho}\left[\left\|\frac{\grad h_\rho(x)}{h_\rho(x)} - \frac{\grad f_\rho(x)}{f_\rho(x)} \right\|^2 \right] \lesssim e^{-\frac{ \rho^2}{2\eps^2(1+\rho^2)}} \cdot \left(1  + m_2^2 + \sup_j \|\mu_j\|^2\right)+ \sum_{j > 0} e^{-\frac{j^2 \rho^2}{2 \eps^2(1+\rho^2)}}\frac{j}{\eps(1+\rho^2)}
    \end{align*}
    where $m_2^2 = \E_{x \sim h_\rho}\left[\|x\|^2 \right]$.
\end{lemma}

\begin{proof}
    We have that
    \begin{align*}
        h_\rho(x) = \sum_{i=1}^k \beta_j g^j_\rho (x - \mu_j)
    \end{align*}
    where $g_\rho^j(x) = g^j(x) * w_{\rho}(x)$. 
    By Lemma~\ref{lem:smoothed_discrete_univariate}, we have that for every $i, j$,
    \begin{align*}
        &\left|\left(\grad g_\rho^j(x) \right)_i - \left(\grad w_{\sqrt{\rho^2+1}}(x) \right)_i \right|\\
        &\lesssim d e^{-\frac{\rho^2}{2 \eps^2(1+\rho^2)}} \left|\left(\grad w_{\sqrt{\rho^2+1}}(x) \right)_i \right| + d  \cdot \sum_{j > 0} e^{-\frac{j^2 \rho^2}{2 \eps^2(1+\rho^2)}} \cdot \frac{j}{\eps(1+\rho^2)} \cdot w_{\sqrt{\rho^2+1}}(x)
    \end{align*}
    So,
    \begin{align*}
        &\left|\left(\grad h_\rho(x) \right)_i - \left(\grad f_\rho(x) \right)_i \right|\\
        &= \left|\sum_{j=1}^k \beta_j \cdot \left(\grad g_\rho^j(x-\mu_j)  - \grad w_{\sqrt{\rho^2+1}}(x-\mu_j) \right)_i \right|\\
        &\lesssim d  \sum_{j=1}^k \beta_j \cdot \left(e^{-\frac{ \rho^2}{2 \eps^2(1+\rho^2)}}\left|\grad \left(w_{\sqrt{\rho^2+1}}(x - \mu_j) \right)_i \right| + \sum_{j > 0} e^{-\frac{j^2\rho^2}{2\eps^2(1+\rho^2)}} \frac{j}{\eps(1+\rho^2)} \cdot w_{\sqrt{\rho^2+1}}(x)\right)\\
    \end{align*}
Similarly, for the density, by Lemma~\ref{lem:smoothed_discrete_univariate}
\begin{align*}
    \left| g_\rho^j(x)  - w_{\sqrt{\rho^2 + 1}}(x)  \right| \lesssim d e^{-\frac{ \rho^2}{2 \eps^2(1+\rho^2)}} w_{\sqrt{\rho^2+1}}(x)
\end{align*}
So,
\begin{align*}
    \left|h_\rho(x) - f_\rho(x) \right| &= \left|\sum_{j=1}^k \beta_j \cdot \left(g_\rho^j(x-\mu_j) - w_{\sqrt{\rho^2+1}}(x-\mu_j) \right) \right|\\
    &\lesssim d e^{-\frac{\rho^2}{2 \eps^2(1+\rho^2)}} \sum_{j=1}^k \beta_j \cdot w_{\sqrt{\rho^2+1}}(x - \mu_j)\\
    &\lesssim d e^{-\frac{\rho^2}{2 \eps^2(1+\rho^2)}} f_\rho(x)
\end{align*}
    Thus, we have
    \begin{align*}
        &\E_{x \sim h_\rho}\left[\left(\frac{(\grad h_\rho(x))_i}{h_\rho(x)} - \frac{(\grad f_\rho(x))_i}{f_\rho(x)} \right)^2\right]\\
        &\le \E_{x \sim h_\rho}\left[\left(\frac{\grad h_\rho(x))_i }{f_\rho(x) \cdot \left(1 + O\left(de^{- \frac{ \rho^2}{2 \eps^2(1+\rho^2)}} \right) \right)} - \frac{\left(\grad f_\rho(x) \right)_i}{f_\rho(x)} \right)^2 \right]\\
        &\lesssim d e^{-\frac{\rho^2}{\eps^2(1+\rho^2)}} \cdot \E\left[\left(\frac{\left(\grad f_\rho(x)\ \right)_i}{f_\rho(x)}\right)^2 \right]\\ 
        &+d\cdot \E\left[\left(\frac{\sum_{j=1}^k \beta_j \cdot \left( e^{-\frac{\rho^2}{2\eps^2(1+\rho^2)}}\left|\grad w_{\sqrt{\rho^2+1}}(x-\mu_j)\right|_i + \sum_{j > 0} e^{-\frac{j^2 \rho^2}{2 \eps^2(1+\rho^2)}} \frac{j}{\eps(1+\rho^2)} \cdot w_{\sqrt{\rho^2+1}}(x) \right)}{f_\rho(x) } \right)^2 \right]\\
        &\lesssim \frac{d}{\rho^2} e^{-\frac{\rho^2}{\eps^2(1+\rho^2)}} + d \sum_{j > 0} e^{-\frac{j^2 \rho^2}{ \eps^2 (\rho^2+1)}}\frac{j}{\eps(1+\rho^2)}  + d e^{-\frac{\rho^2}{\eps^2(1+\rho^2)}} \cdot \E\left[\left(\frac{\sum_{j=1}^k \beta_j \cdot |x - \mu_j|_i \cdot w_{\sqrt{\rho^2+1}}(x-\mu_j)}{f_\rho(x)} \right)^2 \right]\\
        &\lesssim d e^{-\frac{\rho^2}{\eps^2(1+\rho^2)}}+d  \sum_{j > 0} e^{-\frac{j^2 \rho^2}{\eps^2(1+\rho^2)}} \frac{j}{\eps(1+\rho^2)}  + d e^{-\frac{\rho^2}{\eps^2(1+\rho^2)}} \cdot \E\left[\sup_j |x - \mu_j|_i^2 \right]\\
        &\lesssim d e^{-\frac{\rho^2}{\eps^2(1+\rho^2)}}\left( 1 + \E\left[x_i^2 \right] + \sup_j |\mu_j|_i^2 \right) + d \sum_{j>0} e^{-\frac{j^2 \rho^2}{\eps^2(1+\rho^2)}} \frac{j}{\eps(1+\rho^2)}  
    \end{align*}
    Thus, we have
    \begin{align*}
        \E_{x \sim h_\rho}\left[\left\|\frac{\grad h_\rho(x)}{h_\rho(x)} - \frac{\grad f_\rho(x)}{f_\rho(x)} \right\|^2 \right] &\lesssim d^2 e^{-\frac{ \rho^2}{\eps^2(1+\rho^2)}} \cdot \left(1  + \E\left[\|x\|^2 \right] + \sup_j \|\mu_j\|^2\right)+ d^2 \sum_{j > 0} e^{-\frac{j^2 \rho^2}{2\eps^2(1+\rho^2)}}\frac{j}{\eps(1+\rho^2)}\\
        &\lesssim e^{-\frac{ \rho^2}{2\eps^2(1+\rho^2)}} \cdot \left(1  + m_2^2 + \sup_j \|\mu_j\|^2\right)+  \sum_{j > 0} e^{-\frac{j^2 \rho^2}{2 \eps^2(1+\rho^2)}}\frac{j}{\eps(1+\rho^2)}
    \end{align*}
    since $\frac{\rho^2}{\eps^2(1+\rho^2)} > C \log d$
\end{proof}
\begin{corollary}
\label{cor:smoothing_unconditional_gives_mixture_of_gaussians}
    Let $d' = O(d)$, and let $g$ be as defined in Definition~\ref{def:unconditional_distribution}, and let $g_\rho = g * \mathcal N(0, \rho^2 I_{d+d'})$ be the $\rho$-smoothed version of $g$. Let $f_\rho$ be the mixture of $(d+d')$-dimensional standard Gaussians, given by
    \begin{align*}
        f_\rho(y) = \frac{1}{2^d} \sum_{x \in \{\pm 1\}^d} w_{\sqrt{\rho^2+1}}(y - R \cdot \wt x)
    \end{align*}
    where $\wt x \in \R^{d+d'}$ has the first $d$ coordinates given by $x$, and the last $d'$ coordinates $0$.
    Then, for $\frac{\rho^2}{\eps^2(1+\rho^2)} > C \log d$ for sufficiently large constant $C$, we have
    \begin{align*}
        \E_{x \sim g_\rho}\left[\left\| \grad \log g_\rho(x) - \grad \log f_\rho(x) \right\|^2 \right] \lesssim e^{-\frac{\rho^2}{2 \eps^2(1+\rho^2)}} \cdot \left(1  + R^2 + \rho^2 \right) + \sum_{j > 0} e^{-\frac{j^2 \rho^2}{2\eps^2(1+\rho^2)}} \frac{j}{\eps(1+\rho^2)}
    \end{align*}
\end{corollary}
\begin{proof}
    Follows from the facts that $m_2^2 \lesssim d \left(R^2 + \rho^2\right)$ and $\mu_j^2 \lesssim d R^2$ for all $j$.
\end{proof}


%% file: relu_approx_large_smoothing.tex
\subsection{ReLU network for Score at Large smoothing Level}
This section shows how to represent the score of the $\sigma$-smoothed unconditional distribution defined in Definition~\ref{def:unconditional_distribution} for large $\sigma$ using a ReLU network with a polynomial number of parameters bounded by a polynomial in the relevant quantities. We proceed in two stages -- first, we show how to represent the score of a mixture of Gaussians placed on the vertices of a scaled hypercube. Then, we show that for large $\sigma$, this network is close to the score of the $\sigma$-smoothed unconditional distribution.

\begin{lemma}[ReLU network representing score of mixture of Gaussians on hypercube]
    \label{lem:gaussian_mixture_relu}
     For any $\sigma > 0$ and $R > 1$ consider the distribution on $\R^d$ with pdf
     \begin{align*}
         f_\sigma(x) = \frac{1}{2^d} \sum_{\mu \in \{\pm 1\}^d} w_\sigma(x - R \mu)
     \end{align*}
     where $w_\sigma$ is the pdf of $\mathcal N(0, \sigma^2 I_d)$.
     
    There is a constant depth ReLU network $h: \R^d \to \R^d$ with $O\left(\frac{d R}{\gamma^3 \sigma^4} \right)$ parameters, with absolute values bounded by $O\left(\frac{d R}{\sigma^3 \gamma^2} \right)$ such that
    \begin{align*}
        \E_{x \sim f_\sigma}\left[\|\grad \log f_\sigma(x) - h(x)\|^2 \right] \lesssim \gamma^2
    \end{align*}
\end{lemma}
\begin{proof}
    Note that $g_\sigma$ is a product distribution. So, the claim follows by Lemma~\ref{lem:product_score_relu}.
\end{proof}
\begin{lemma}
    \label{lem:relu_large_smoothing}
    Let $d' = O(d)$, and let $R \le \poly(d)$. Let $g$ be the pdf of the unconditional distribution on $\R^{d+d'}$, as defined in Definition~\ref{def:unconditional_distribution}, and let $g_\sigma$ be its $\sigma$-smoothed version with score $s_\sigma$. For $\eps < \frac{1}{C \sqrt{\log d}}$,  and $\sigma > C \eps \left(\sqrt{\log d} + \sqrt{\log \frac{1}{\eps}} \right)$ for sufficiently large constant $C$, there is a constant depth ReLU network $h$ with $O\left( \frac{d R}{\gamma^3 \sigma^4} \right)$ parameters with absolute values bounded by $O\left(\frac{d R}{\sigma^3 \gamma^2} \right)$ such that
    \begin{align*}
        \E_{x \sim g_\sigma}\left[\|s_\sigma(x) - h(x) \|^2\right] \lesssim \gamma^2 + \frac{1}{d^{C^2/20}}
    \end{align*}
\end{lemma}
\begin{proof}
    Let $h$ be the ReLU network from Lemma~\ref{lem:gaussian_mixture_relu} for smoothing $\sigma$. It satisfies our bounds on the number of parameters and the absolute values. 

    Note that for our setting of $\eps$ and $\sigma$, we have that
    \begin{align*}
        \frac{\sigma^2}{\eps^2 (1 + \sigma^2)} &= \frac{1}{\eps^2\left(1 + \frac{1}{\sigma^2} \right)} > \frac{1}{\eps^2 + \frac{1}{C^2 \log d}} > \frac{1}{\frac{2}{C^2 \log d}} > \frac{C^2 \log d}{2}
    \end{align*}
    and
    \begin{align*}
        \frac{\sigma^2}{\eps^2(1+\sigma^2)} > \frac{C^2 (\log d + \log \frac{1}{\eps})}{2} > \log \frac{1}{\eps(1+\sigma^2)}
    \end{align*}
    
    So, by Lemma~\ref{cor:smoothing_unconditional_gives_mixture_of_gaussians}, for the mixture of Gaussians $f_\sigma$ as described in Lemma~\ref{lem:gaussian_mixture_relu}, for $R \le \poly(d)$,
    \begin{align*}
        \E_{x \sim g_\sigma}\left[\|s_\sigma(x) - \grad \log f_\sigma(x) \|^2 \right] \lesssim e^{-\frac{\sigma^2}{10 \eps^2(1+\sigma^2)}} \lesssim \frac{1}{d^{C^2/20}}
    \end{align*}
    Also, by Lemma~\ref{lem:smoothed_discrete_univariate},
    \begin{align*}
        |g_\sigma(x) - f_\sigma(x)| \lesssim \frac{f_\sigma(x) }{d^{C^2/20}}
    \end{align*}
    So, by Lemma~\ref{lem:gaussian_mixture_relu},
    \begin{align*}
        \E_{x \sim g_\sigma}\left[\|\grad \log f_\sigma(x) - h(x)\|^2 \right] \lesssim \E_{x \sim f_\sigma}\left[\|\grad \log f_\sigma(x) - h(x)\|^2 \right] \lesssim \gamma^2
    \end{align*}
    So we have
    \begin{align*}
        \E_{x \sim g_\sigma}\left[\|s_\sigma(x) - h(x)\|^2 \right] \lesssim \gamma^2 + \frac{1}{d^{C^2/20}}
    \end{align*}
\end{proof}

%% file: relu_approx_unconditional.tex
\subsection{ReLU Network Approximating score of Unconditional Distribution}
\begin{restatable}[ReLU Score Approximation for Lower bound Distribution]{theorem}{relu_approx_final}
    \label{thm:unconditional_relu}
    Let $C$ be a sufficiently large constant, and let $d' = O(d)$. Fix any $\sigma \ge \tau$ for $\tau = \frac{1}{d^C}$. Given a constant-depth ReLU network representing a function $f : \{-1, 1\}^{d} \to \{-1, 1\}^{d'}$ with $\poly(d)$ parameters, there is a constant-depth ReLU network $h : \R^{d+d'} \to \R^{d+d'}$ with $\poly\left(d\right)$ parameters with each parameter bounded in absolute value by $\poly\left(d\right)$ such that for the unconditional distribution $g$ defined in Definition~\ref{def:unconditional_distribution} with $\sigma$-smoothed version $g_\sigma$ and corresponding score $s_\sigma$, for $R > C \log d$, $\frac{1}{\poly(d)} < \eps < \frac{1}{C \sqrt{\log d}}$,
    \begin{align*}
        \E_{x \sim g_\sigma} \left[\|s_\sigma(x) - h(x)\|^2 \right] \lesssim \frac{1}{d^{C/200}}
    \end{align*}     
\end{restatable}
\begin{proof}
    Follows by Lemmas~\ref{lem:relu_small_smoothing}~and~\ref{lem:relu_large_smoothing}.
\end{proof}
\lowerboundwellmodeled*

\begin{proof}
    Follows via reparameterization from the Theorem, and rescaling.
\end{proof}

%% file: alpha_well_behaved.tex

%
%
%
%
%
%
%

%% file: lower_bound_putting_it_together.tex
\section{Lower Bound -- Putting it all Together}
\lowerbound*
\begin{proof}
    First, by Lemma \ref{cor:one_way_relu}, there exists a ReLU network that represents a one-way function $f:\{\pm 1\}^m \to \{\pm 1\}^m$, with constant weights, polynomial size, and parameters bounded in magnitude by $\text{poly}(d)$.
        
    Therefore, by Corollary \ref{cor:g_is_well_modeled}, the distribution $\wt g$ over $\R^{d}$ is a $C$-well-modeled distribution, if we take $R = C \log d$, $\eps = \frac{1}{C\sqrt{\log d}}$. Further, if we take a linear measurement model with $\beta = \frac{1}{C^2\log^{2}(d)}$, then by Lemma \ref{lem:lower_bound_plugging_params_generic}, any $(1/10, 1/10)$-posterior sampler for this distribution takes at least $2^{\Omega(m)}$ time to run.
\end{proof}

\lowerboundfinegrained*
\begin{proof}
    First, by Lemma \ref{cor:one_way_relu}, there exists a ReLU network that represents a one-way function $f:\{\pm 1\}^m \to \{\pm 1\}^m$, with constant weights, polynomial size, and parameters bounded in magnitude by $\text{poly}(d)$.
        
    Therefore, by Corollary \ref{cor:g_is_well_modeled}, the distribution $\wt g$ over $\R^{d}$ is a $C$-well-modeled distribution, if we take $R = C \log d$, $\eps = \frac{1}{C\sqrt{\log d}}$. Further, if we take a linear measurement model with $\beta = \frac{1}{C^2\log^{2}(d)}$, then by Lemma \ref{lem:lower_bound_plugging_params}, any $(1/10, 1/10)$-posterior sampler for this distribution takes at least $2^{\Omega(m)}$ time to run.
\end{proof}

%% file: upper_bound.tex
\section{Upper Bound}
\label{sec:upper_bound_proof}



\begin{lemma}
    \label{lem:rejection_sampling_success_probability}
    Let $q$ be a distribution over $\R^m$ such that $\E_{w \sim q}[{\|w\|_2^2}] = O(m)$.
    Let $w \sim q$ and $y = w + \beta \mathcal{N}(0, I_m)$. Then, there exists a constant $c > 0$ such that \[
        \Prb[y]{\Prb[w]{\norm{y - w} \le 10\gamma\sqrt{m+\log(1 / \delta)} \ \middle |\  y} \ge (c\gamma)^m \cdot \delta^{m / 2 + 1}} \ge 1 - \delta.
    \]
\end{lemma}

\begin{proof}
    Since $\E_{w \sim q}[{\|w\|_2^2}] \lesssim m$, there exists a constant $C$ such that \[
        \Prb[w \sim q]{\|w\|_2^2 > \frac{Cm}{\delta}} < \frac{\delta}{3}. 
    \]

    Lemma~\ref{lem:covering_number} shows that there exists a covering over  $\{x \in \R^m \mid \norm{x}_2 \le \sqrt{Cm / \delta}\}$ with $N = O(\frac{1}{\sqrt{\delta}\beta})^m$ balls of radius $\beta \sqrt{m + \log(1 / \delta)}$. Let $S$ be the set of all the covering balls. This means that \[
        \Pr[\exists \theta \in S : w \in \theta] \ge 1 - \frac{\delta}{3}.
    \] Define \[
        S' := \{\theta \in S \mid \Pr_w[w \in \theta] > \frac{\delta}{3N}\}.
    \]
    Then we have that with high probability, $w$ will land in one of the cells in $S'$:
    \begin{align*}
        \Prb[w]{\forall \theta \in S' : w \notin\theta} \le \Pr[\forall \theta \in S : w \notin \theta] + \Pr[{
\bigvee_{\theta \in S \setminus S'} w \in \theta}] \le \frac{\delta}{3} + N \cdot \frac{\delta}{3N} \le \frac{2\delta}{3}.
    \end{align*}
    Moreover, we define \[
        S^+ := \{y \in \R^m \mid \exists \theta \in S', \forall w \in \theta : \norm{w - y} \le 10\beta\sqrt{m + \log \frac{1}{\delta}}\}.
    \]
    By the sampling process of $y$, we have that \begin{align*}
        \Prb[y]{y \in S^+} &= \Prb[w\sim q, z \sim \mathcal{N}(0,  I_m)]{w + \beta z \in S^+}  \\
        &\ge \Prb[w\sim q, z \sim \mathcal{N}(0, I_m)]{(\exists \theta \in S': w \in \theta )\wedge (\norm{z} \le 8\sqrt{m})}  \\
        &\ge 1 - \Prb[w]{\forall \theta \in S' : w \notin\theta} -   \Prb[z \sim \mathcal{N}(0, I_m)]{\norm{z}^2 > 64 (m + \log \frac{1}{\delta})}  
    \end{align*}

    By Lemma~\ref{lem:chi_squared_concentration}, we have \[
        \Prb[z \sim \mathcal{N}(0, I_m)]{\norm{z}^2 > 64 (m + \log \frac{1}{\delta})} < \frac{\delta}{3}.
    \]
    Therefore, \[
        \Prb[y]{y \in S^+} \ge 1 - \delta.
    \]
    This implies that with $1 - \delta$ probability over $y$, there exists a cell $\theta \in S$ such that $\norm{{y} - \tilde{t}} \le 10\beta$ and $\Pr_w[w \in \theta] \ge \frac{\delta}{3N} \ge
    \delta \cdot \Theta(\sqrt{\delta}\beta)^m$. 
    

\end{proof}

\begin{lemma}
    \label{lem:unconditonal_implies_conditional}
    Consider a well-modeled distribution and a linear measurement model.
    Suppose we have a $(\tau, \delta)$-unconditional sampler for the distribution, where $\tau < \frac{c\delta\beta}{\sqrt{m + \log (1 / \delta)}}$ for a sufficiently small constant $c > 0$.
    Then rejection sampling (Algorithm~\ref{alg:rej_sampling}) gives a $(\tau, 2\delta)$-posterior sampler using at most $\frac{\log (1 / \delta)}{\delta^2}(\frac{O(1)}{\beta \sqrt{\delta}})^m$ samples .
\end{lemma}

\begin{proof}


     Let $\mathcal{P}$ be the distribution that couples true distribution $\mathcal{D}$ over $(x, y)$ and the output distribution of the posterior sampler $\wh{p}_{|y}$. Rigorously, we define $\mathcal{P}$ over $(x, \wh x, y) \in X \times X\times Y$ with density $p^\mathcal{P}$ such that $p^\mathcal{P}(x,  y) = p^\mathcal{D}(x, y)$, $p^\mathcal{P}(\wh x \mid y) = \wh p_{\mid y}(\wh x)$. Similarly,
     we let $\wt{\mathcal{P}}$ over $(x, \wh x, y) \in \R^d \times \R^d \times \R^\alpha$ be the joint distribution between the unconditional sampler over $(x, \wh x)$ and the measurement process $\mathcal{D}$ over $(x, y)$. Then by the definition of unconditional samplers, we have \[
        \Prb[x, \hat x \sim \wt{\mathcal{P}}]{\norm{x - \wh
        x} \ge \tau} \le \delta.
     \]
     Therefore, to prove the correctness of the algorithm, we only need to show that there exists a $\wh{\mathcal{P}}$ over $(x, \wh x, y)$ such that $\wt{\mathcal{P}}(\wh x \mid y) = \wh{p}_{\mid y}(\wh x)$ and $\TV(\wh{\mathcal{P}}, \wt{\mathcal{P}}) \le \delta$.
     By Lemma~\ref{lem:chi_squared_concentration}, \[
        \Prb[\wt{\mathcal{P}}']{\norm{Ax - y}^2 \ge 4\beta^2 (m + \log \frac{1}{\delta})} \le \frac{\delta}{4}.
    \]
    Therefore, we define $\wt{\mathcal{P}}'$ as $\wt{\mathcal{P}}$ conditioned on $\norm{x - \wh x} < \tau$ and $\frac{\norm{Ax - y}^2}{2\beta^2} \le 2(m + \log \frac{1}{\delta})$.
    Then we have \[
        \TV(\wt{\mathcal{P}}, \wt{\mathcal{P}}') \le \frac{3\delta}{2}.
    \] 
    \paragraph{Algorithm correctness.}

     We have
     \begin{align*}
         \wh p_{\mid y}(\wh x) = \frac{p^{\wt{\mathcal{P}}'}(\wh x) \cdot e^\frac{-\|A{\wh x} - y\|^2}{2\beta^2}}{\int p^{\wt{\mathcal{P}}'}(\wh x) \cdot e^\frac{-\|A{\wh x} - y\|^2}{2\beta^2} \d {\wh x}} 
         = \frac{\int p^{\wt{\mathcal{P}}'}(x, \wh x) \cdot e^\frac{-\|A{\wh x} - y\|^2}{2\beta^2} \d x}{\int p^{\wt{\mathcal{P}}'}(\wh x) \cdot e^\frac{-\|A{\wh x} - y\|^2}{2\beta^2} \d {\wh x}},
 \end{align*}
     Then we define \[
        r(\wh x) := \frac{\int p^{\wt{\mathcal{P}}'}(x, \wh x) \cdot e^\frac{-\|A{x} - y\|^2}{2\beta^2} \d x }{\int p^{\wt{\mathcal{P}}'}(x, \wh x) \cdot e^\frac{-\|A{\wh x} - y\|^2}{2\beta^2} \d x }.
     \]
     Conditioned on $\norm{x - \wh x} \le \tau$ and $\frac{\norm{Ax - y}^2}{2\beta^2} \le 2(m + \log \frac{1}{\delta})$, we have \begin{align*}
         \abs{\log r(\wh x)} &\le \sup_x \frac{|{\norm{Ax - y}^2 - \norm{A\wh x - y}^2}|}{2\beta^2} \\
         &\le \frac{\tau^2\norm{A}_2^2 + 2\tau\norm{A}_2\norm{Ax - y}}{2\beta^2} \\
         &\lesssim \frac{\tau^2}{\beta^2} + \frac{\tau  \sqrt{m + \log (1 / \delta)}}{\beta}. 
     \end{align*}
    By our setting of $\tau$, we have $1 - \delta/8 < r(\wh x) < 1 + \delta / 8$.

     So we have 
     \begin{align*}
         \int p^{\wt{\mathcal{P}}'}(\wh x) \cdot e^\frac{-\|A{\wh x} - y\|^2}{2\beta^2} \d {\wh x} = \int p^{\wt{\mathcal{P}}'}(x, \wh x) \cdot e^\frac{-\|A{\wh x} - y\|^2}{2\beta^2} \d x \d {\wh x} = \left(1 \pm \frac{\delta}{8}\right) \int p^{\wt{\mathcal{P}}'}(x, \wh x) \cdot e^\frac{-\|A{x} - y\|^2}{2\beta^2} \d x \d {\wh x}.
     \end{align*}

     Hence,
     \begin{align*}
        \wh{ p}_{\mid y}(\wh x) &= \frac{r(\wh x) \cdot \int p^{\wt{\mathcal{P}}'}(x, \wh x)  e^\frac{-\|A{ x} - y\|^2}{2\beta^2} \d x }{(1 \pm \frac{\delta}{8})  \int p^{\wt{\mathcal{P}}'}(x, \wh x)  e^\frac{-\|A{x} - y\|^2}{2\beta^2} \d x \d {\wh x}} \\
            &= \frac{r(\wh x) \cdot \int p^{\wt{\mathcal{P}}'}(x, \wh x) p^{\wt{\mathcal{P}}'}(y \mid x)  \d x }{(1 \pm \frac{\delta}{8}) p^{\wt{\mathcal{P}}'}(y)} \\&= \left(1 \pm \frac{\delta}{4}\right) r(\wh x) \int p^{\wt{\mathcal{P}}'}(x, \wh x \mid y) \d x \\
            &= \left(1 \pm \frac{\delta}{2}\right) p^{\wt{\mathcal{P}}'} (\wh x \mid y).
     \end{align*}
     Finally, we have \begin{align*}     
        \int \abs{p^{\wt{\mathcal{P}}'}(\wh x \mid y) - \wh{ p}_{\mid y}(\wh x)} \d {\wh x}  \d p^{\wt{\mathcal{P}}'}(y) &= \int \abs{\left(1 \pm \frac{\delta}{2}\right) p^{\wt{\mathcal{P}}'} (\wh x \mid y) - p^{\wt{\mathcal{P}}'} (\wh x \mid y)}\d {\wh x} \d p^{\wt{\mathcal{P}}'}(y)
        \le \frac{\delta}{2}.
     \end{align*}
    This implies that \[
        \TV(\wh{\mathcal{P}}_{\wh x}, \wt{\mathcal{P}}_{\wh x}) = \TV(\wh{\mathcal{P}}_{\wh x}, \wt{\mathcal{P}}'_{\wh x}) + \TV(\wt{\mathcal{P}}'_{\wh x}, \wt{\mathcal{P}}_{\wh x}) \le \frac{\delta}{4} + \frac{\delta}{2}  \le \frac{3\delta}{4}.
    \] 
    Hence, \[
        \Prb[x, \wh{x} \sim \wh{ \mathcal{P}}]{\norm{x - \wh x} \ge \tau} \le \frac{3\delta}{4} + \delta \le \frac{7\delta}{4}.
    \]

    \paragraph{Running time.}  
    Now we prove that for most $y$ 
    For $y \in Y$,
    for each round, the acceptance probability $q(y)$ each round is that
    \begin{align*}
    q(y) &=
        \int p^{\wt{\mathcal{P}}'}(\wh x) e^{-\frac{\norm{y - A\wh{x}}^2}{2 \gamma^2}} \d {\wh x} \\
        &=  (1 \pm \frac{\delta}{8}) \int p^{\wt{\mathcal{P}}'}(x, \wh x) \cdot e^\frac{-\|A{x} - y\|^2}{2\beta^2} \d x \d {\wh x} \\
        &\ge \frac{1}{2} \int p^{\mathcal{X}}(x) \cdot e^\frac{-\|A{x} - y\|^2}{2\beta^2} \d x \\
         &= \frac{1}{2} \Ex[ x \sim \mathcal{X}]{e^{-\frac{\norm{Ax - y}^2}{2 \beta^2}}} \\
         &\ge \frac{1}{2} \Prb[x \sim \mathcal{X}]{\norm{Ax - y} \le 10\sqrt{m + \log (1 / \delta)}\beta} \cdot e^{-\frac{100(m + \log (1 / \delta))\beta^2}{2\beta^2}} \\
         &= \frac{1}{2} \Prb[x \sim \mathcal{X}]{\norm{Ax - y} \le 10\sqrt{m + \log (1 / \delta)}\beta} \cdot \delta e^{-50 m}
    \end{align*}

    By Lemma~\ref{lem:bounded_norm_after_projection},  $\E_{x \sim \mathcal{X}}[{\norm{Ax}_2^2}] = O(m)$.
    By Lemma~\ref{lem:rejection_sampling_success_probability}, we have that for $1 - \delta/8$ probability over $y$, for some $c > 0$, \[
        \Prb[x \sim \mathcal{X}]{\norm{Ax - y} \le 10\sqrt{m + \log (1 / \delta)}\beta} \ge (c\beta)^m \cdot \delta^{m / 2 + 1}.
    \]
    Therefore, for some $c > 0$, \[
        \Prb[y \sim \mathcal{Y}]{q(y) \ge (c\beta)^m \cdot \delta^{m / 2 + 2}} \ge 1 - \frac{\delta}{8}.
    \]
    Hence, for some $C > 0$, \[
        \Prb{\text{Rejection sampling terminates in } \frac{\log (1 / \delta)}{\delta^2}\left(\frac{C}{\beta \sqrt{\delta}}\right)^m \text{ rounds}} \ge 1 - \frac{\delta}{4}.
    \]
\end{proof}

\upperbound*

\begin{proof}
    Theorem~\ref{thm:unconditional_sampler} suggests that for an $O(C)$-well-modeled distribution, a $\poly(d)$ time $(\frac{1}{d^{3C}}, \frac{1}{2d^{C}})$-unconditional sampler exists. Since \[
        \frac{1}{d^{3C}} < o\left(\frac{\frac{1}{2d^C} \cdot \frac{1}{d^{C}}}{\sqrt{d}}\right) < o\left(\frac{\delta \beta^2}{\sqrt{m + \log (1 / \delta)}}\right).
    \] By lemma~\ref{lem:unconditonal_implies_conditional}, a  $(\frac{1}{d^{3C}}, \frac{1}{d^C})$-posterior sampler exists using $\frac{\log (1 / \delta)}{\delta^2}(\frac{O(1)}{\beta \sqrt{\delta}})^m \le \poly(d)(\frac{O(1)}{\beta \sqrt{\delta}})^m$ samples. Since generating each sample costs $\poly(d)$ time. The total time is $\poly(d)(\frac{O(1)}{\beta \sqrt{\delta}})^m$. 
\end{proof}

%% file: unconditional_sampler.tex
\section{Well-Modeled Distributions Have Accurate Unconditional Samplers}
\label{appendix:unconditional_sampling}

\paragraph{Notation.} For the purposes of this section, we let $\wt s_t = s_{\sigma^2}$ denote the score at time $t$.

\begin{definition}[Forward and Reverse SDE]
    For distribution $q_0$ over $\R^d$, consider the Variance Exploding (VE) Forward SDE, given by
    \begin{align*}
        d x_t =  d B_t, \quad x_0 \sim q_0
    \end{align*}
    where $B_t$ is Brownian motion, so that $x_t \sim x_0 + \mathcal N(0, t I_d)$. Let $q_t$ be the distribution of $x_t$. 

    There is a VE Reverse SDE associated with the above Forward SDE given by
    \begin{align}
        \label{eq:reverse_SDE}
        d x_{T-t} = \wt s_{T-t}(x_{T-t}) + d B_t
    \end{align}
    for $x_T \sim q_T$.
\end{definition}
\begin{theorem}[Unconditional Sampling Theorem, Implied by \cite{benton2024nearly}, adapted from \cite{gupta2023sampleefficient}]
\label{thm:unconditional_sampling_generic}
    Let $q$ be a distribution over $\R^d$ with second moment $m_2^2 = \E_{x \sim q}\left[ \|x\|^2\right]$ between $\frac{1}{\poly(d)}$ and $\poly(d)$. Let $q_t = q * \mathcal N(0, t I_d)$ be the $\sqrt{t}$-smoothed version of $q$, with corresponding score $\wt s_t$. Suppose $T = d^C$. For any $\gamma > 0$, there exist $N = \wt O\left( \frac{d}{\eps^2} \log^2 \frac{1}{\gamma} \right)$ discretization times $0= t_0 < \dots < t_N \le T-\gamma$ such that, given score approximations $h_{T-t_k}$ of $\wt s_{T-t_k}$ that satisfy
    \[
        \E_{x \sim q_{T - t_k}}\left[\|\wt s_{T - t_k} - h_{T - t_k}\|^2 \right] \lesssim \frac{ \eps^2}{C \cdot (T-t_k) \cdot \log \frac{d}{\gamma}}
    \]
    for sufficiently large constant $C$, then, the discretization of the VE Reverse SDE defined in \eqref{eq:reverse_SDE} using the score approximations can sample from a distribution $\eps + \frac{1}{d^{C/2}}$ close in $TV$ to a distribution $\gamma m_2$-close in $2$-Wasserstein to $q$ in $N$ steps.
\end{theorem}

\unconditionalsampler*
\begin{proof}
    The definition of a well-modeled distribution gives that, for every $\frac{1}{d^C} < \sigma < d^C$ there is an approximate score $\wh s_\sigma$ such that
    \begin{align*}
        \E_{x\sim p_\sigma}\left[\|\wh s_\sigma(x) - s_\sigma(x)\|^2 \right] < \frac{1}{d^C \sigma^2}
    \end{align*}
    and $\wh s_\sigma$ can be computed by a $\poly(d)$-parameter neural network with $\poly(d)$ bounded weights. Here $p_\sigma$ is the $\sigma$-smoothed version of $p$ with score $s_\sigma$.

    Then, by Theorem~\ref{thm:unconditional_sampling_generic}, this means that the discretized reverse diffusion process can use the $\wh s_\sigma$ to produce a sample $\wh x$ from a distribution $\wh p$ that is $\frac{1}{d^{C/3}}$ close in TV to a distribution $\frac{1}{d^{C/3}}$ close in 2-Wasserstein. This means there exists a coupling between $\wh x \sim \wh p$ and $x \sim p$ such that
    \begin{align*}
        \Pr\left[\|\wh x - x\| > \frac{1}{d^{C/6}} \right] < \frac{1}{d^{C/6}}
    \end{align*}
    The claim follows via reparameterization.
\end{proof}

%% file: crypto_hardness.tex
\section{Cryptographic Hardness} 
\label{sec:crypto}

Recall that a one-way function $f$ is a function such that every polynomial-time algorithm fails to find a pre-image of a random output of $f$ with high 
probability. 


\begin{lemma}
\label{lem:one_way_function_stretching}
    If a one-way function $f : \{\pm 1\}^n \to \{\pm 1\}^{m(n)}$ exists, then for any $\frac{1}{\poly(n)} \le l(n) \le \poly(n)$, there exists a one-way function $g : \{\pm 1\}^n \to \{\pm 1\}^{l(n)}$.
\end{lemma}
\begin{proof}
    For $l(n) > m(n)$, we just need to pad $l(n) - m(n)$ 1's at the end of the output, i.e., \[
        g(x) := (f(x), 1^{l(n) - m(n)}).
    \] 
    For $\frac{1}{\poly(n)} \le l(n) < m(n)$, for each $n$, there exists a constant $c < 1$ such that $l(n) = m(n^c)$. Then we can satisfies the requirement by defining \[
        g(x) := f(\text{first $n^c$ bits of }x).
    \]
\end{proof}


\begin{lemma}
    Every circuit $f: \{\pm1\}^n \to \{\pm1\}^{m(n)}$ of $\poly(n)$ size can be simulated by a ReLU network with $\poly(n)$ parameters and constant weights.
\end{lemma}
\begin{proof}
    In the realm of $\{+1, -1\}$, $-1$ corresponds to True and $+1$ corresponds to False. We can use a layer of neurons to translate it to $\zo$ first, where $1$ corresponds to True and $-1$ corresponds to False. We will translate $\zo$ back to $\{+1, -1\}$ when output.
    
    Now we only need to show that the logic operation ($\neg$, $\wedge$, $\vee$) in each gate of the circuit can be simulated by a constant number of neurons with constant weights in ReLU network when the input is in $\zo^n$:
    \begin{itemize}
        \item 
    For each AND ($\wedge$) gate, we use $\ReLU(\sum (y_i - 1) + 1)$ to calculate $\bigwedge y_i$.

    \item For each
    OR ($\vee$) gate, we use  
    $\ReLU(1 - \ReLU(1 - \sum y_i))$ to calculate $\bigvee y_i$. 

    \item For each NOT ($\neg$) gate , we use $\ReLU(1 - y_i)$ to calculate $\neg y_i$.
    \end{itemize}

    It is easy to verify that for $\zo$ input, the output of each neuron-simulated gate will remain in $\zo^n$ and equal to the result of the logical operation.
\end{proof}

Then the next corollary directly follows.
\begin{corollary}
\label{cor:one_way_relu}
    Every one-way function can be computed by a ReLU network with $\poly(n)$ parameters, and constant weights. 
\end{corollary}

%% file: utility.tex
\section{Utility Results}
\begin{lemma}
\label{lem:score_derivative_second_moment}
Let $p_\sigma$ be some $\sigma$-smoothed distribution with score $s_\sigma$. For any $\varepsilon \le \sigma$, 
\begin{align*}
        \E_{x \sim p_\sigma} \sup_{\abs{c} \le \eps} s'_\sigma(x+c)^2 \lesssim \frac{1}{\sigma^4}
    \end{align*}
\end{lemma}
\begin{proof}
Draw $x \sim p_\sigma$, and let $z \sim N(0, \sigma^2)$ be independent of $x$. By Lemma \ref{lem:a1gupta},
  \[
    s_\sigma(x) = \E_{z \mid x}\left[\frac{z}{\sigma^2}\right].
  \]
  Moreover, by Corollary \ref{cor:usable_score_derivative},
  \[
    s_\sigma(x+c) = \frac{\mathbb{E}_{z \mid x}\left[e^{\frac{2c z - c^2}{2\sigma^2}} \left(\frac{z-c}{\sigma^2}\right)\right]}{\mathbb{E}_{z \mid x}\left[e^{\frac{2c z - c^2}{2\sigma^2}}\right]} = \frac{\E_{z \mid x}\left[e^{c z/\sigma^2} \left(\frac{z-c}{\sigma^2}\right)\right]}{\E_{z \mid x}[e^{c z/\sigma^2}]}
  \]
  Taking the derivative with respect to $c$, since $(a/b)' = (a'b - ab')/b^2$,
  \begin{align}
    s_\sigma'(x+c) 
        &= \frac{\E_{z \mid x}\left[e^{c z/\sigma^2 } \left(\frac{z^2 - zc - \sigma^2}{\sigma^4}\right)\right]\E_{z \mid x}[e^{c z/\sigma^2 }] - \E_{z \mid x}\left[e^{c z /\sigma^2} \left(\frac{z-c}{\sigma^2}\right)\right]\E_{z \mid x}\left[\frac{z}{\sigma^2} e^{c z/\sigma^2 }\right] }{\E_{z \mid x}[e^{c z/\sigma^2 }]^2}\nonumber\\
        &= \frac{\E_{z \mid x}\left[e^{c z/\sigma^2 } \left(\frac{z^2 - \sigma^2}{\sigma^4}\right)\right]\E_{z \mid x}[e^{c z/\sigma^2 }] - \E_{z \mid x}\left[e^{c z/\sigma^2 } \frac{z}{\sigma^2}\right]^2}{\E_{z \mid x}[e^{c z /\sigma^2}]^2}\nonumber
        \\ &\leq \frac{\E_{z \mid x}[e^{\eps z/\sigma^2} \frac{z^2}{\sigma^4}]}{\E_{z \mid x}[e^{\eps z/\sigma^2}]}\label{eq:sprimebound}
  \end{align}
    Now we take the supremum over all $\abs{c} \le \eps$, and take the expectation of this quantity over $x$ to get the desired moment: 
  \begin{align}
    \E_x \left[\sup_{\abs{c} \le \eps} s'_\sigma(x+c)^2\right] 
        &\le \E_x \left[\sup_{\abs{c} \le \eps} \frac{\E_{z \mid x}[e^{c z/\sigma^2} \frac{z^2}{\sigma^4}]^2}{\E_{z \mid x}[e^{c z/\sigma^2}]^2}\right]\nonumber
        \\&\le \E_x \left[\left(\sup_{\abs{c} \le \eps} \E_{z \mid x}\left[e^{c z/\sigma^2} \frac{z^2}{\sigma^4}\right]^2\right)\left(\sup_{\abs{c} \le \eps}\E_{z \mid x}\left[e^{c z/\sigma^2}\right]^{-2}\right)\right]\nonumber
        \\&\le \sqrt{\E_x \left[\sup_{\abs{c} \le \eps} \E_{z \mid x}\left[e^{c z/\sigma^2} \frac{z^2}{\sigma^4}\right]^4\right]\E_x \left[\sup_{\abs{c} \le \eps}\E_{z \mid x}\left[e^{c z/\sigma^2}\right]^{-4}\right]}\label{eq:score_moment_intermediate_bound}
  \end{align}
  The last inequality here follows from Cauchy-Schwarz. For the first term of equation \ref{eq:score_moment_intermediate_bound}, we have
  \[\E_x \left[\sup_{\abs{c} \le \eps} \E_{z \mid x}\left[e^{c z/\sigma^2} \frac{z^2}{\sigma^4}\right]^4\right]\leq \E_x \E_{z \mid x}\left[(e^{\eps z/\sigma^2} + e^{-\eps z/\sigma^2}) \frac{z^2}{\sigma^4}\right] \eqqcolon g(x)\]
  We compute the 4th moment of this term directly:
  \begin{align}
    \E_x[g(x)^4] &= \E_x\left[\E_{z \mid x}\left[(e^{\eps z/\sigma^2} + e^{\eps z/\sigma^2}) \frac{z^2}{\sigma^4}\right]^4\right]\nonumber\\
                 &\leq \E_z\left[(e^{\eps z/\sigma^2} + e^{-\eps z/\sigma^2})^4 \frac{z^{8}}{\sigma^{16}}\right]\nonumber\\
                 &\leq \sqrt{\E_z[(e^{\eps z/\sigma^2} + e^{-\eps z/\sigma^2})^{8}] \E_z\left[\frac{z^{16}}{\sigma^{32}}\right]}\nonumber\\
                 &\leq \sqrt{\E_z[2^{8}(e^{8 \eps z/\sigma^2} + e^{-8\eps z/\sigma^2})] \E_z\left[\frac{z^{16}}{\sigma^{32}}\right]}\nonumber\\
                 &\lesssim \sqrt{e^{32 \eps^2/\sigma^2} \cdot \frac{1}{\sigma^{16}}} =  \frac{e^{16 \eps^2/\sigma^2}}{\sigma^{8}}\label{eq:numerator_bound}
  \end{align}
  For the second term of equation \ref{eq:score_moment_intermediate_bound},
  \begin{align}
      \E_x \left[\sup_{\abs{c} \le \eps}\E_{z \mid x}\left[e^{c z/\sigma^2}\right]^{-4}\right] &\le \E_x \left[\sup_{\abs{c} \le \eps}\E_{z \mid x}\left[e^{-4c z/\sigma^2}\right]\right] && \text{ by Jensen's}\nonumber
      \\&\le \E_x \left[\E_{z \mid x}\left[e^{4\varepsilon \abs{z}/
      \sigma^2}\right]\right]\nonumber
      \\&\le \E_z \left[e^{4\varepsilon z/\sigma^2} + e^{-4\varepsilon z/\sigma^2}\right]\nonumber
      \\&= 2e^{\frac{1}{2}\sigma^2\cdot (4\eps/\sigma^2)^2} =2e^{8  \eps^2/\sigma^2}\label{eq:denom_bound}
  \end{align}
  So, putting equations \ref{eq:denom_bound} and \ref{eq:numerator_bound} into equation \ref{eq:score_moment_intermediate_bound}, we get
  \begin{align*}
    \E_x \left[\sup_{\abs{c} \le \eps} s'_\sigma(x+c)^2\right] 
        &\le \sqrt{2e^{8  \eps^2/\sigma^2} \cdot \frac{e^{16 \eps^2/\sigma^2}}{\sigma^{8}}}
        \end{align*}
    Now, by assumption, $\eps \le \sigma$. So, we finally get that  
    \begin{align*}
    \E_x \left[\sup_{\abs{c} \le \eps} s'_\sigma(x+c)^2\right] 
        &\lesssim \sqrt{\frac{1}{\sigma^{8}}} = \frac{1}{\sigma^4}
        \end{align*}
\end{proof}
\begin{lemma}
    \label{lem:max_score_interval_bounded}
    Let $p$ be a distribution over $\R$, and let $p_\sigma = p*N(0, \sigma^2)$ have score $s_\sigma$. If $\gamma \le \sigma/4$, then, 
    \[\Pr\left[\sup_{y \in [x-\gamma, x+\gamma]} s(y) \ge t\right] \le e^{-\sigma t}\]
\end{lemma}
\begin{proof}
    From Corollary \ref{cor:score_interval_moment_bound_simplified}, we have
    \[\E\left[\sup_{y \in [x - \gamma, x + \gamma]}s(y)^k\right] \le \frac{k^k15^k}{\sigma^k}\]
    So, we have
    \[\Pr\left[\sup_{y \in [x - \gamma, x + \gamma]} s(x)^k \ge t^k\right] \le \frac{\E\left[\sup_{y \in [x - \gamma, x + \gamma]} s(x)^k\right]}{t^k} \le \left(\frac{15k}{t\sigma }\right)^k\]
    Setting $k = \log \frac{1}{\delta}$, we get 
    \[\Pr\left[\sup_{y \in [x - \gamma, x + \gamma]} |s(x)| \ge \frac{15}{e\sigma}\log\frac{1}{\delta}\right] \le \delta\]
\end{proof}
For the following Lemmas, If $p$ is a distribution over $\R$ and has score $s$, define the Fisher information $\mathcal{I}$ as
\[\mathcal{I}\coloneqq \E_{x \sim p} [s^2(x)]\]

\begin{lemma}[Lemma A.1 from \cite{gupta2022finitesample}]
\label{lem:a1gupta}
    Let $p$ be a distribution over $\R$, and let $p_\sigma = p * N(0, \sigma^2)$ have score $s_\sigma$. Let $(x, y, z)$ be the joint distribution such that $y \sim p$, $z \sim N(0, \sigma^2)$ are independent, and $x = y + z$. For all $\eps > 0$, 
    \[\frac{p(x + \eps)}{p(x)} = \E_{z \mid x}\left[e^{\frac{2\eps z - \eps^2}{2\sigma^2}} \right]\text{ and } s_\sigma(x) = \E_{z \mid x} \left[\frac{z}{\sigma^2}\right]\]
\end{lemma}
\begin{corollary}\label{cor:usable_score_derivative} Let $p$ be a distribution over $\R$, and let $p_\sigma = p * N(0, \sigma^2)$ have score $s_\sigma$.
\[
    s_\sigma(x+\eps)= \frac{\E_{z \mid x}\left[e^{\eps z/\sigma^2} \left(\frac{z-\eps}{\sigma^2}\right)\right]}{\E_{z \mid x}[e^{\eps z/\sigma^2}]}
  \]    
\end{corollary}
\begin{proof}
This proof is given in Lemma A.2 of \cite{gupta2022finitesample}, and is reproduced here for convenience and completeness, since a statement in the middle of their proof is what we use.

By Lemma \ref{lem:a1gupta}, we have 
\[\frac{p_\sigma(x + \eps)}{p_\sigma(x)} = \E_{z \mid x}\left[e^{\frac{2\eps z - \eps^2}{2\sigma^2}} \right]\]
Taking the derivative with respect to $\eps$, we have 
\[\frac{p'_\sigma(x + \eps)}{p_\sigma(x)} = \mathbb{E}_{z \mid x}\left[e^{\frac{2\eps z - \eps^2}{2\sigma^2}} \left(\frac{z-\eps}{\sigma^2}\right)\right]\]
So,
\begin{align*}
    s_\sigma(x+\eps) 
    &= \frac{p_\sigma'(x+\eps)}{p_\sigma(x+\eps)} = \frac{p_\sigma'(x+\eps)}{p_\sigma(x)}\frac{p_\sigma(x)}{p_\sigma(x+\eps)}
    \\&=\frac{\mathbb{E}_{z \mid x}\left[e^{\frac{2\eps z - \eps^2}{2\sigma^2}} \left(\frac{z-\eps}{\sigma^2}\right)\right]}{\mathbb{E}_{z \mid x}\left[e^{\frac{2\eps z - \eps^2}{2\sigma^2}}\right]} = \frac{\E_{z \mid x}\left[e^{\eps z/\sigma^2} \left(\frac{z-\eps}{\sigma^2}\right)\right]}{\E_{z \mid x}[e^{\eps z/\sigma^2}]}
\end{align*}
\end{proof}
\begin{lemma}[Lemma 3.1 from \cite{gupta2022finitesample}]
\label{lem:fisher_bound}
    Let $p$ be a distribution over $\R$ and let $p_\sigma = p *N(0,\sigma^2)$ have Fisher information $\mathcal{I}_\sigma$. Then, $\mathcal{I}_\sigma \le \frac{1}{\sigma^2}$.
\end{lemma}
\begin{lemma}[Lemma B.3 from \cite{gupta2022finitesample}]
\label{lem:score_second_moment_bound}
    Let $p$ be a distribution over $\R$ and let $p_\sigma = p *N(0,\sigma^2)$ have score $s_\sigma$ and Fisher information $\mathcal{I}_\sigma$. If $\abs{\gamma} \le \sigma/2$, then 
    \[\E[s^2(x + \gamma)] \le \mathcal{I}_\sigma + O\left(\frac{\gamma}{\sigma}\mathcal{I}_\sigma \sqrt{\log \frac{1}{\sigma^2\mathcal{I}_\sigma}}\right)\]
\end{lemma}
\begin{lemma}[Lemma A.6 from \cite{gupta2022finitesample}]
\label{lem:score_interval_moment_bound}
    Let $p$ be a distribution over $\R$ and let $p_\sigma = p *N(0,\sigma^2)$ have score $s_\sigma$ and Fisher information $\mathcal{I}_\sigma$. Then, for $k \ge 3$ and $\abs{\gamma} \le \sigma/2$, 
    \[\E[\abs{s_\sigma(x + \gamma)}^k] \le \frac{k!}{2}(15/\sigma)^{k-2}\max\left(\E[s^2_\sigma(x +\gamma)], \mathcal{I}_\sigma \right)\]
\end{lemma}

\begin{corollary}
\label{cor:score_interval_moment_bound_simplified}
    Let $p$ be a distribution over $\R$ and let $p_\sigma = p *N(0,\sigma^2)$ have score $s_\sigma$. Then, for $k \ge 3$ and $\abs{\gamma} \le \sigma/2$, 
    \[\E[\abs{s_\sigma(x + \gamma)}^k] \le \left(\frac{15k}{\sigma}\right)^k\]
\end{corollary}
\begin{proof}
    Consider the continuous function $f(x) = x \sqrt{\log \frac{1}{\sigma^2x}}$. This function is only defined on $0 < x \le 1/\sigma^2$. We have
    \[f'(x) = \frac{2\log \frac{1}{\sigma^2 x}-1}{2\sqrt{\log \frac{1}{\sigma^2 x}}}.\]
    Setting this equal to zero gives $x = \frac{1}{\sigma^2\sqrt{e}}$. $f(\frac{1}{\sigma^2\sqrt{e}}) = \frac{1}{\sigma^2\sqrt{2e}}$. Since $f(1/\sigma^2) = 0$ and $\lim_{x \to 0^+} f(x) = 0$, we have this is the maximum value of the function. Further, we know by Lemma \ref{lem:fisher_bound} that $I_\sigma \le 1/\sigma^2$. So, along with the fact that $\abs{\gamma} \le \sigma/2$, we have
    \[\frac{\gamma}{\sigma}I_\sigma \sqrt{\log \frac{1}{\sigma^2I_\sigma}} \lesssim \frac{1}{\sigma^2} \]
    Therefore, from Lemma \ref{lem:score_second_moment_bound}, and using Lemma \ref{lem:fisher_bound} again, we get
    \[\E[s^2(x + \gamma)] \le \mathcal{I}_\sigma + O\left(\frac{\gamma}{\sigma}I_\sigma \sqrt{\log \frac{1}{\sigma^2I_\sigma}}\right) \lesssim \frac{1}{\sigma^2}\]
    Finally, we can plug this into Lemma \ref{lem:score_interval_moment_bound} to get
    \begin{align*}
    \E[\abs{s_\sigma(x + \gamma)}^k] &\le \frac{k!}{2}(15/\sigma)^{k-2}\max\left(\E[s^2_\sigma(x +\gamma)], I_\sigma \right) 
    \\&\lesssim k^k\frac{15^{k-2}}{\sigma^{k-2}}\cdot \frac{1}{\sigma^2} \le k^k\frac{15^k}{\sigma^k}
    \end{align*}
\end{proof}

\begin{lemma}[Laurent-Massart Bounds\cite{LaurentMassart}]
    \label{lem:chi_squared_concentration}
    Let $v \sim \mathcal{N}(0, I_n)$. For any $t > 0$, 
        \[\Pr[\norm{v}^2 - n \geq 2\sqrt{nt} + 2t] \leq e^{-t}.\]
\end{lemma}

\begin{lemma}[See \citet{mohri2018foundations}, Lemma~6.27]
    \label{lem:covering_number}
     There exist $\Theta(R/\eps)^d$ $d$-dimensional balls of radius $\eps$ that cover $\{x \in \R^d \mid \norm{x}_2 \le R\}$.
\end{lemma}

\begin{lemma}
    \label{lem:bounded_norm_after_projection}
    Let $p$ be a distribution over $\R^d$ with covariance $\Sigma$ such that $\norm{\Sigma} \lesssim 1$, and let $A \in \R^{m \times d}$ be a matrix with $\norm{A} \le 1$. Then \[
        \E_{x \sim p}[\norm{Ax}^2] \lesssim m.     
    \]
\end{lemma}

\begin{proof}
    Note the expectation of the squared norm $\|Ax\|^2$ can be expressed as:
    \[
    \mathbb{E}_{x \sim p}[\|Ax\|^2] = \text{trace}(A^T A \Sigma).
    \]
    
    Given that $\|A\| \leq 1$, the singular values of $A$ are at most 1. Hence, the matrix $A^T A$, which represents the sum of squares of these singular values, will have its trace (sum of eigenvalues) bounded by $m$:
    \[
    \text{trace}(A^T A) \leq m.
    \]
    
    Hence, given that $\|\Sigma\| \lesssim 1$, we have :
    \[
    \mathbb{E}_{x \sim p}[\|Ax\|^2] = \text{trace}(A^T A \Sigma) \le \norm{\Sigma} \cdot \text{trace}(A^T A) \lesssim m.
    \]
\end{proof}